\documentclass{article}

\PassOptionsToPackage{numbers, compress}{natbib}

 \usepackage[preprint]{neurips_2026}

\usepackage[utf8]{inputenc} 
\usepackage[T1]{fontenc}    
\usepackage{hyperref}       
\usepackage{url}            
\usepackage{booktabs}       
\usepackage{amsfonts}       
\usepackage{nicefrac}       
\usepackage{microtype}      
\usepackage{xcolor}         

\usepackage{amsmath}
\usepackage{amssymb}
\usepackage{mathtools}
\usepackage{amsthm}
\usepackage[acronym]{glossaries}
\usepackage{tabularx}
\usepackage{multirow}
\usepackage{tikz}
\usepackage{graphicx}
\usepackage{titletoc}
\usepackage{float}
\usetikzlibrary{arrows.meta, positioning, calc, fit, backgrounds, shapes.geometric}                                                 
\usepackage{pifont,colortbl}

\usepackage[capitalize,noabbrev]{cleveref}

\newacronym{auroc}{AUROC}{area under the receiver operating characteristic curve}
\newacronym{bce}{BCE}{binary cross-entropy}
\newacronym{bn}{BN}{batch normalization}
\newacronym{cbm}{CBM}{concept bottleneck model}
\newacronym{ce}{CE}{cross-entropy}
\newacronym{celeba}{\texttt{CelebA}}{CelebFaces Attributes}
\newacronym{cmnist}{\texttt{CMNIST}}{Colored MNIST}
\newacronym{cnc}{CnC}{correct-n-contrast}
\newacronym{cobalt}{CoBalT}{concept balancing technique}

\newacronym{dfr}{DFR}{deep feature reweighting}
\newacronym{gdro}{group DRO}{group distributionally robust optimization}
\newacronym{eiil}{EIIL}{environment inference for invariant learning}
\newacronym{erm}{ERM}{empirical risk minimization}
\newacronym{exmap}{ExMap}{explainability-guided pseudo-group map}
\newacronym{isic}{\texttt{ISIC}}{International Skin Imaging Collaboration}
\newacronym{jtt}{JTT}{just train twice}
\newacronym{moe}{MoE}{mixture of experts}
\newacronym{nct}{NCT}{neural classification tree}
\newacronym{pwga}{pWGA}{pseudo worst-group accuracy}
\newacronym{umnist}{\texttt{UMNIST}}{Undersampled MNIST}
\newacronym{waterbirds}{\texttt{Waterbirds}}{Waterbirds}
\newacronym{wga}{WGA}{worst group accuracy}
\glsunset{waterbirds}
\newacronym{lrp}{LRP}{Layer-wise Relevance Propagation}

\glsdisablehyper

\theoremstyle{plain}
\newtheorem{theorem}{Theorem}[section]

\theoremstyle{definition}
\newtheorem{definition}[theorem]{Definition}
\newtheorem{assumption}[theorem]{Assumption}
\theoremstyle{remark}

\title{Discovering Latent Groups for Robust Classification}

\author{Ankur Garg, Ulrich Aïvodji, Samira Ebrahimi Kahou, Vincent Michalski}

\begin{document}

\maketitle

\begin{abstract}
Machine learning models exploit spurious correlations, achieving high average accuracy but failing disproportionately on underrepresented subgroups. Existing methods address this by adjusting network parameters, guided either by subgroup annotations or inferred pseudo-group labels. Yet at inference, these methods produce only a class prediction, with no insight into a sample's latent subgroup. We propose neural classification trees (NCT), a framework that achieves robustness by encoding subgroup structure in its tree-shaped architecture. By routing each sample to an ``easy'' or ``hard'' node of this tree---based on prediction correctness---and reusing these routes as pseudo-labels for the next iteration, NCT disentangles conflicting subgroups, without requiring subgroup supervision. We evaluate NCT on five benchmarks spanning binary and multi-class spurious correlations. Our experiments show that the learned tree topology provides strong interpretability by consistently isolating minority subgroups, which provides a transparent mapping between the model architecture and the data's latent group structure, while yielding competitive robustness with state-of-the-art methods.~\footnote{\url{https://github.com/agarg-dev/Neural-Classification-Trees/}}
\end{abstract}

\section{Introduction}

Deep neural networks are prone to learning shortcuts or features that are correlated with labels in the training set but have no causal relationship to them.~\citep{Geirhos2020ShortcutLI}. A classic example is \gls{waterbirds} dataset, where a standard model will classify birds based on the background (\emph{water} vs. \emph{land}) rather than the bird's features. When such spurious correlations shift at test time, performance degrades for minority subgroups, such as waterbird in front of a land background.

Although the research community has developed effective methods to mitigate this issue, a critical gap still remains: interpretability of the latent subgroup structure. Supervised approaches such as \gls{gdro}~\citep{sagawa2019distributionally}, while robust, require expensive, fine-grained annotations for every training sample. \Gls{jtt}~\citep{liu2021just}, \gls{eiil}~\citep{creager2021environment}, and \gls{dfr}~\citep{kirichenko2022last} drop training-time labels but still rely on validation group labels, while GEORGE~\citep{sohoni2020no} and \gls{exmap}~\citep{chakraborty2024exmap} operate without any group annotation. Across all three regimes, the model's parameters are adjusted to improve worst-group metrics, but the final classifier remains opaque at inference: it does not reveal which latent subgroup a specific sample belongs to, nor does it structurally isolate conflicting features.

We take a different approach: instead of only relying on parameter updates to handle diverse groups, we make the partition architectural. We introduce NCT, a framework that builds on a well-established observation: training difficulty is a reliable proxy for semantic group identity, with samples aligned with spurious correlations being easy and conflicting counter-examples being hard. NCT iteratively routes samples to easy and hard branches based on this signal, and --- unlike methods that use the same signal transiently --- preserves the resulting partition as the inference-time architecture itself.

Our contributions are: (i) \textbf{Architecture as partition.} Where prior work~\citep{liu2021just, zhang2022correct, pezeshki2024discovering, chakraborty2024exmap} uses correctness, loss, or attribution signals transiently for reweighting or pseudo-group inference, \gls{nct} preserves the difficulty-based partition as the inference-time architecture, with each leaf encoding both the predicted class and the difficulty path. (ii) \textbf{Unsupervised depth selection.} A \gls{pwga} criterion with a Wilson-tolerance early-stopping bound decides when to stop deepening the tree without group annotations. (iii) \textbf{Theoretical motivation.} Iteration-1 errors are enriched with minority samples under simplicity bias (Theorem~\ref{thm:main}); structural separation yields a positive minimax-risk gap under feature conflict (Theorem~\ref{thm:approx}). (iv) \textbf{Empirical evidence.} Across five benchmarks, NCT concentrates minority subgroups in hard branches (82\% of landbird-on-water, 73\% of blond-male, 47\% of benign-no-patch, 71\% of digit-8 and 84\% of color-conflicting digits) while matching or approaching state-of-the-art worst-group accuracy against eight baselines.

\section{Related Work} 

\paragraph{Optimization for Spurious Correlations.}
A dominant paradigm for mitigating spurious correlations modifies the optimization objective to upweight minority groups. When group labels are available, \gls{gdro}~\citep{sagawa2019distributionally} minimizes the worst-case group loss. When training labels are unavailable, recent approaches infer pseudo-groups to guide optimization. GEORGE~\citep{sohoni2020no} clusters representations from a standard \gls{erm} model to estimate subclass labels, subsequently using them for robust optimization. \Gls{jtt}~\citep{liu2021just} identifies error sets from an early-stopped \gls{erm} model and retrains the final model by upweighting these hard samples. \Gls{cnc}~\citep{zhang2022correct} extends this by applying contrastive learning to push the inferred groups apart in feature space. Other works focus on the classifier head; \gls{dfr}~\citep{kirichenko2022last} freezes the feature extractor and retrains only the last layer on a group-balanced validation set. 

\paragraph{Inferring Environments and Pseudo-Groups.}
A parallel line of work treats group identification as an inference problem in its own right. \Gls{eiil}~\citep{creager2021environment} learns a soft environment partition that maximally violates an invariance penalty, which is then handed to invariant or distributionally robust optimizers. LfF~\citep{nam2020learning} reweights a debiased model using the loss of a deliberately biased one, while SelecMix~\citep{hwang2022selecmix} exploits bias-conflicting pairs. More recent work sharpens the partition: XRM~\citep{pezeshki2024discovering} trains twin networks on disjoint halves of the training data and uses confident held-out cross-mistakes to discover environments without group-annotated validation data, and GIC~\citep{han2024gic} infers groups from a spurious-attribute classifier whose predictions vary across distributional shifts. These methods share a downstream pattern: the discovered partition is consumed by a separate invariant-learning or reweighting algorithm, and the resulting classifier is opaque at inference. \Gls{nct} is orthogonal: its contribution is the architectural persistence of the partition, rather than a sharper pseudo-group signal.

\paragraph{Input-Level and Explainability-Guided Interventions.}
Rather than altering the loss function, some methods intervene directly on the data or leverage post-hoc explanations. MaskTune~\citep{taghanaki2022masktune} forces the model to explore robust features by masking out the most discriminative regions of the input image during fine-tuning. \Gls{exmap}~\citep{chakraborty2024exmap} clusters explainability heatmaps from a pre-trained model into pseudo-groups, exploiting the observation that spurious and robust decisions produce distinct attribution patterns. \Gls{nct} contributes to this neighborhood by using correctness rather than attributions as the difficulty signal, and by preserving the resulting partition as the inference-time architecture rather than as a one-off clustering step.

\paragraph{Structural and Modular Interpretability.}
A complementary line of work builds interpretability directly into the model architecture. \Glspl{cbm}~\citep{koh2020concepts} force information to pass through a layer of human-aligned concepts but typically require expensive concept annotations or discovery algorithms like \gls{cobalt}~\citep{arefin2024unsupervised}. Tree-structured classifiers such as ProtoTree~\citep{nauta2021neural} and NBDT~\citep{wan2020nbdt} also expose inference-time decision paths, but through class-taxonomy hierarchies; neither was designed for spurious-correlation settings. \Gls{nct} adopts a modular approach akin to \gls{moe}~\citep{shazeer2017outrageously, jacobs1991adaptive} or neural trees~\citep{tanno2019adaptive}. Unlike standard \gls{moe} that routes samples to maximize predictive likelihood---and therefore collapses to spurious shortcuts---\gls{nct} routes by correctness, decomposing the spurious-vs-core subgroup structure.

Across these four families, \gls{nct} is the only entry whose inference-time architecture exposes the spurious-correlation subgroup partition: each leaf records both the predicted class and the difficulty path that produced it, with no post-hoc clustering required.

\begin{figure}[t]
    \centering
    \includegraphics[width=0.9\linewidth]{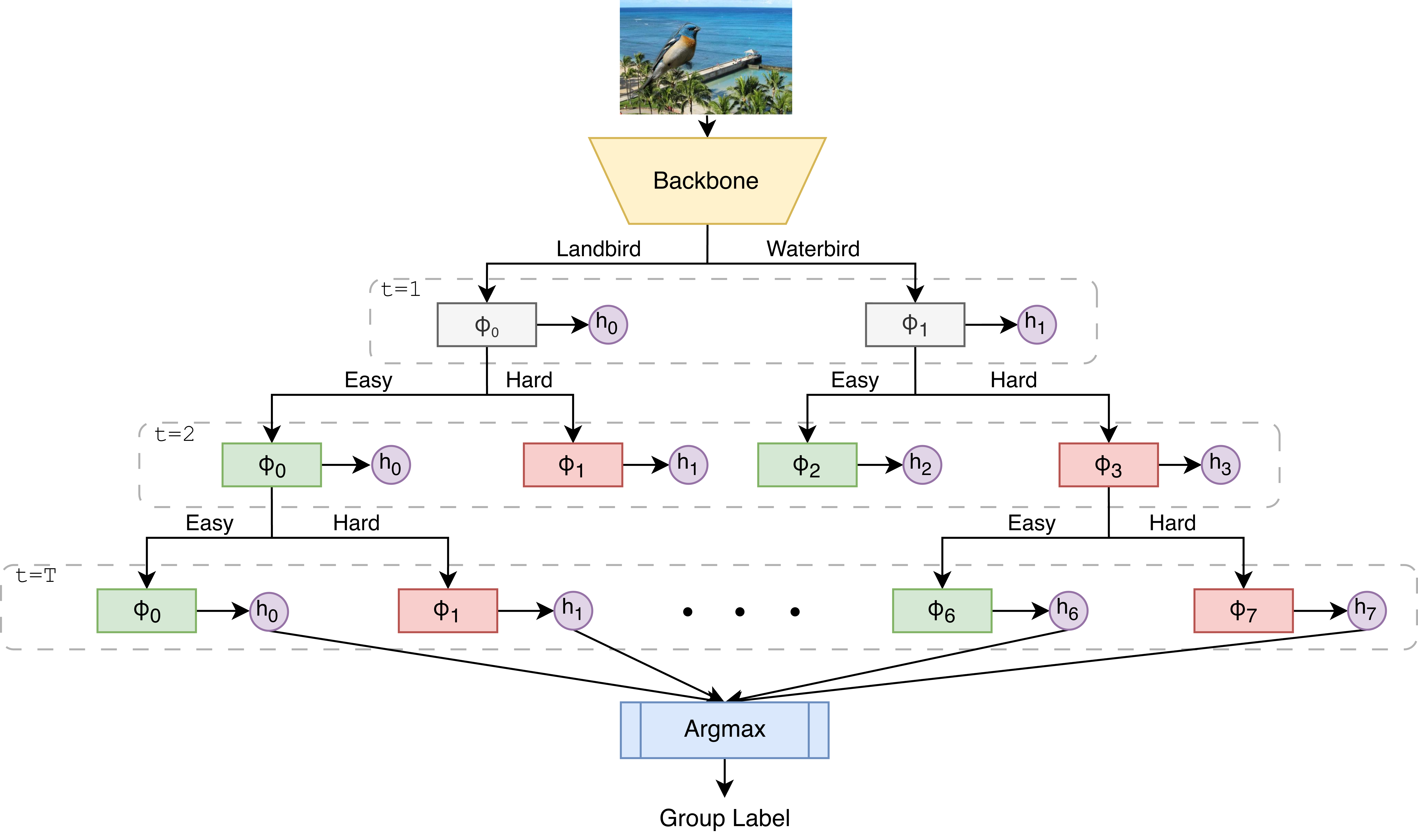}
    \caption{\textbf{\gls{nct} inference.} Backbone features propagate through parent-to-child head connections across iterations. The argmax over all leaf node outputs determines the group label, encoding both class and difficulty path.}
    \label{fig:architecture}
\end{figure}

\begin{figure}[t]
    \centering
    \resizebox{\textwidth}{!}{%
    \begin{tikzpicture}[
        >={Stealth[length=2.5mm]},
        node distance=0.8cm and 1.2cm,
        font=\small\sffamily,
        thick,
        box/.style={
            rectangle, rounded corners=2pt, draw=gray!80,
            minimum width=2cm, minimum height=1cm, align=center, fill=white
        },
        modelbox/.style={ box, fill=gray!15, draw=gray!80 },
        matchbox/.style={
            shape=diamond, aspect=1, draw=gray!80,
            minimum width=1.8cm, minimum height=1.8cm, align=center,
            fill=gray!10, inner sep=0pt
        },
        arrow/.style={->, >=Stealth, draw=gray!70, thick, rounded corners=3pt},
        labelnode/.style={
            font=\scriptsize\sffamily, fill=white, inner sep=1.5pt, text=gray!90!black
        },
        groupbox/.style={
            draw=gray!60, dashed, inner sep=10pt, rounded corners=4pt, fill=gray!5
        }
    ]

    \node[box, fill=blue!10, draw=blue!60!black] (sample) at (0,0) {\textbf{Sample}};

    \node[modelbox] (backbone) at (3.5, 1.5) {Backbone};
    \node[modelbox] (parents) [right=of backbone] {Parent\\[-2pt]Heads};
    \node[modelbox] (heads)   [right=of parents]  {K Heads};
    \node[modelbox] (probs)   [right=of heads]    {K Probs};
    \node[modelbox] (argmax)  [right=of probs]    {argmax};
    
    \node[box, fill=orange!10, draw=orange!60!black] (predicted) [right=of argmax] {\textbf{Predicted}};

    \node[box, fill=blue!15, draw=blue!60!black] (assigned) at (3.5, -1.5) {\textbf{Assigned}\\[-2pt]{\scriptsize(node $\ell$)}};

    \node[matchbox] (match) at (23.5, 0) {Match?};

    \node[box, fill=teal!10, draw=teal!60!black] (easy) at (27.5, 1.2) {$2\ell$ (easy)};
    \node[box, fill=red!10, draw=red!60!black]  (hard) at (27.5, -1.2) {$2\ell{+}1$ (hard)};

    \begin{scope}[on background layer]
        \node[groupbox, fit=(backbone) (heads), label={[anchor=south, yshift=2pt]north:\textbf{NCT Architecture}}] (nct_group) {};
    \end{scope}

    \draw[arrow] (sample.east) -- ++(0.8, 0) coordinate(fork);
    \draw[arrow] (fork) |- (backbone.west);
    \draw[arrow] (fork) |- (assigned.west);

    \draw[arrow] (backbone.east) -- (parents.west);
    \draw[arrow] (parents.east)  -- (heads.west);
    \draw[arrow] (heads.east)    -- (probs.west);
    \draw[arrow] (probs.east)    -- (argmax.west);
    \draw[arrow] (argmax.east)   -- (predicted.west);

    \coordinate (turn_point) at ($(match.west) + (-1.0, 0)$);
    \draw[arrow] (predicted.east) -| (turn_point) -- (match.west);
    \draw[arrow] (assigned.east)  -| (turn_point) -- (match.west);

    \draw[arrow] (match.east) -- ++(0.8, 0) |- (easy.west) node[pos=0.75, labelnode] {Yes};
    \draw[arrow] (match.east) -- ++(0.8, 0) |- (hard.west) node[pos=0.75, labelnode] {No};

    \end{tikzpicture}%
    }
    \caption{\textbf{Routing mechanism.} A sample's current node assignment $\ell_i^{(t)}$ is updated based on prediction correctness: correctly classified samples proceed to the easy branch, while misclassified samples are routed to the hard branch.}
    \label{fig:routing}
\end{figure}
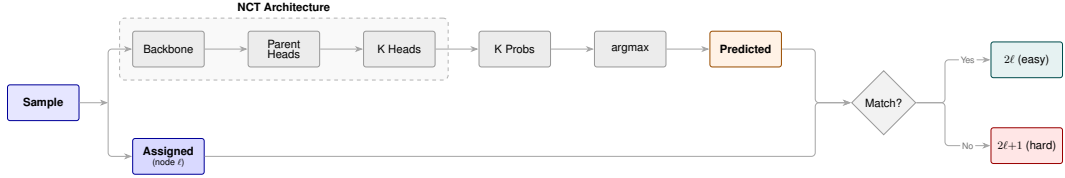

\section{Methodology}

\subsection{Problem Formulation}

We consider a $C$-way classification task over a dataset $\mathcal{D} = \{(x_i, y_i)\}_{i=1}^{N}$, where $x_i \in \mathcal{X}$ is an input image and $y_i \in \{0, \ldots, C{-}1\}$ is the class label. Our objective is to learn a hierarchical mapping that decomposes the data distribution into partitions of varying difficulty without attribute supervision.

The training process proceeds in discrete iterations $t \in \{1, \ldots, T\}$, where the model maintains a set of $K^{(t)} = C \cdot 2^{t-1}$ classification nodes. Each training sample $x_i$ is assigned a unique node index $\ell_i^{(t)} \in \{0, \ldots, K^{(t)}{-}1\}$, with $\ell_i^{(1)} = y_i$. Each node serves as a specialized expert for its assigned samples. Setting $C = 2$ recovers the binary case.

\subsection{Hierarchical Feature Architecture}
\label{sec:architecture}

The \gls{nct} is defined by a shared backbone $f_{\theta}: \mathcal{X} \rightarrow \mathbb{R}^d$ and a hierarchy of classification heads $\mathcal{H}^{(t)}$ as illustrated in \cref{fig:architecture}. Each head, indexed by node $j \in \{0, \ldots, K^{(t)}{-}1\}$, has internal representation $\phi_j^{(t)}(x)$ and output logit $h_j^{(t)}(x)$. For $t=1$, heads receive backbone features $f_\theta(x)$; for $t > 1$, head $j$ receives parent representation $\phi_{\lfloor j/2 \rfloor}^{(t-1)}(x)$. We denote the full output vector as $\mathbf{h}^{(t)}(x) = [h_0^{(t)}(x), \ldots, h_{K^{(t)}{-}1}^{(t)}(x)]$.

For the initial iteration ($t=1$), the heads operate directly on the backbone features $z = f_\theta(x)$. For subsequent iterations ($t > 1$), we implement a \textbf{hierarchical feature flow}. Rather than discarding previous heads, we retain all heads from prior iterations and utilize them as intermediate feature extractors. A specific child head $j$ at iteration $t$ receives its input not from the raw backbone, but from the parent representation $\phi_{\lfloor j/2 \rfloor}^{(t-1)}(x)$. This ensures that specialized heads in deeper layers operate on features that have already been conditioned by the class-specific semantics of the parent.

\paragraph{Asymmetric Head Capacity.} For $t \geq 2$, each split produces an even-indexed easy child ($2j$) and an odd-indexed hard child ($2j{+}1$). Easy children use a linear $\phi_j^{(t)}$; hard children use a linear--ReLU--dropout block. The asymmetry reflects what each branch handles: easy partitions contain samples the parent already classifies correctly and need little added capacity, while hard partitions absorb the misclassified samples whose conflicting features demand richer representations.

\paragraph{Inference.} During inference, backbone features propagate through the hierarchy and all leaf nodes produce outputs in parallel as shown in \cref{fig:architecture}. We define the predicted node assignment as $\hat{\ell}_i^{(t)} = \arg\max_{j} h_j^{(t)}(x_i)$. The final class is read off from the predicted leaf:
\begin{equation}
\hat{y}_i \;=\; \bigl\lfloor \hat{\ell}_i^{(T)} \,/\, 2^{T-1} \bigr\rfloor,
\label{eq:inference}
\end{equation}
which generalizes the binary $\hat{\ell}_i^{(T)} < 2^{T-1}$ rule to $C$ classes. Beyond the class, $\hat{\ell}_i^{(T)}$ also records the easy/hard trajectory along the route to its leaf, exposing the model's structural grouping of the input.

\subsection{Difficulty-Based Sample Partitioning}

We employ a \textbf{correctness-based partitioning} strategy to determine the tree topology (\cref{fig:routing}). This mechanism divides the training data based on the alignment between the predicted and assigned nodes. The routing is deterministic and driven by training dynamics. We define the binary difficulty indicator as $d_i^{(t)} = \mathbb{I}[\hat{\ell}_i^{(t)} \neq \ell_i^{(t)}]$.

The node assignment for the next iteration follows
\begin{equation}
\ell_i^{(t+1)} = 2\,\ell_i^{(t)} + d_i^{(t)}.
\label{eq:routing}
\end{equation}
This update rule guarantees that a parent node $j$ splits into exactly two children: an easy child ($2j$) containing samples correctly classified at step $t$, and a hard child ($2j + 1$) containing misclassified samples. Consequently, the binary string $d_i^{(1)} d_i^{(2)} \cdots d_i^{(T-1)}$ along the route to leaf $\ell_i^{(T)}$ encodes the precise history of training difficulty for the samples assigned to it.

\subsection{Optimization and Stability}
\label{sec:optimization}

Training a hierarchical model iteratively can lead to catastrophic forgetting~\cite{mccloskey1989catastrophic}, where updates for child nodes destabilize the parent representations that children depend on. To mitigate this, we employ a stabilized optimization objective and a two-phase training protocol.

\paragraph{Auxiliary Loss.} To ensure that the feature representations remain valid for the parent tasks, we enforce an auxiliary loss on the parent nodes during child training. The total loss at iteration $t$ is

\begin{equation}
\mathcal{L}_{\text{total}} = \mathcal{L}_{\text{BCE}}\bigl(\sigma(\mathbf{h}^{(t)}(x)),\, \ell^{(t)}\bigr) \\ + \lambda_{\text{aux}}\, \mathcal{L}_{\text{BCE}}\bigl(\sigma(\mathbf{h}^{(t-1)}(x)),\, \ell^{(t-1)}\bigr),
\label{eq:total_loss}
\end{equation}

where $\mathcal{L}_{\text{BCE}}$ is the one-vs-all \gls{bce} computed against the one-hot encoding of the node assignment, $\ell^{(t-1)} = \lfloor \ell^{(t)} / 2 \rfloor$ is the parent assignment, $\mathbf{h}^{(t-1)}(x)$ is the output logits of the parent heads, and $\lambda_{\text{aux}}$ controls the trade-off between child specialization and parent stability. 

\paragraph{Two-Phase Training.} Each iteration is trained in two phases to stabilize feature learning. Phase~1 freezes the backbone and parent layers so the new heads adapt to the existing feature space. Phase~2 fine-tunes the entire network so the backbone can resolve hard nodes that were previously inseparable.

\subsection{Controlling Hierarchy Depth}
\label{sec:depth_control}

A correctness-based partition grows multiplicatively in $T$, but not every dataset benefits from a deeper tree. \Gls{nct} pairs two mechanisms---sparse-node merging and a pseudo-WGA depth criterion---both operating without group annotations.

\paragraph{Sparse-Node Merging.} Routing can produce hard children with very few samples, especially after several splits. We merge any hard child $j$ with training count $n_j < m_{\min}$ into its easy sibling $j-1$, reassigning the corresponding pseudo-labels. At inference, the merged head's logit is masked to $-\infty$, ensuring the $\arg\max$ never selects an unused leaf. Setting $m_{\min}=0$ disables merging.

\paragraph{Pseudo-WGA Criterion.} To decide when to stop deepening, at each iteration $t \geq 2$ we group validation samples by the iteration-2 ancestor of their predicted leaf, $a_i^{(t)} = \lfloor \hat{\ell}_i^{(t)} / 2^{t-2} \rfloor \in \{0, \ldots, K^{(2)}{-}1\}$, and define $\mathrm{pWGA}_2^{(t)}$ as the minimum validation accuracy across these $K^{(2)}$ groups. Evaluating at iteration-2 granularity (rather than at the current depth) keeps each group large enough for reliable accuracy estimates as the tree deepens. We keep iteration $t$ as the operating depth only if its $\mathrm{pWGA}_2$ stays within a Wilson tolerance of the running best:
\begin{equation}
\mathrm{pWGA}_2^{(t)} \;\geq\; \max_{s<t}\mathrm{pWGA}_2^{(s)} \;-\; z_\alpha\sqrt{\tfrac{p^{\star}(1-p^{\star})}{n_{\text{worst}}}},
\label{eq:depth_rule}
\end{equation}
where $p^{\star}$ is the running best, $n_{\text{worst}}$ the worst group's sample count, and $z_\alpha = 1.96$. If the inequality fails we stop deepening and return the iteration with the highest $\mathrm{pWGA}_2$. Appendix~\ref{app:pwga_correlation} verifies that $\mathrm{pWGA}_2$ tracks true \gls{wga} closely across our benchmarks.

\section{Theoretical Analysis}
\label{sec:theory}

We provide theoretical motivation for \gls{nct}'s design by analyzing: (1) when correctness-based routing recovers minority subgroups, and (2) why structural separation can reduce worst-case risk compared to single-model approaches.

\paragraph{Setup.} Consider a classification problem where inputs $x \in \mathcal{X}$ have class labels $y \in \{0,1\}$ and latent attributes $a \in \{0,1\}$ representing spurious features. Let $\rho = P(a = y) > 0.5$ denote the spurious correlation strength. This creates four subgroups: two \emph{majority} groups where $a = y$ (e.g., waterbirds on water) and two \emph{minority} groups where $a \neq y$ (e.g., waterbirds on land). Standard \gls{erm} fails on minority groups because it exploits the spurious correlation.

\subsection{Minority Group Recovery}

 \Gls{nct} routes samples to ``easy'' and ``hard'' branches based on prediction correctness: $d_i^{(t)} = \mathbb{I}[\hat{\ell}_i^{(t)}
   \neq \ell_i^{(t)}]$. Since $\ell_i^{(1)} = y_i$, we analyze when this error-based signal at the first iteration correlates with
  minority group membership.

\begin{assumption}[Simplicity Bias]
\label{assump:simplicity}
After iteration 1 training, the model's predictions satisfy $P(\hat{\ell}^{(1)} = a) = 1 - \epsilon$ for some small $\epsilon \geq 0$.
\end{assumption}

This assumption reflects the empirically observed tendency of neural networks to learn simpler, more prominent features before complex ones \citep{shah2020pitfalls, hermann2020shapes}. When spurious correlations are strong, the spurious attribute provides an easy solution that gradient descent finds early in training.

\begin{theorem}[Minority Enrichment]
\label{thm:main}
Under Assumption~\ref{assump:simplicity}, let $\mathcal{E} = \{i : \hat{\ell}_i^{(1)} \neq \ell_i^{(1)}\}$ denote the set of misclassified samples after iteration 1. The proportion of minority samples in this error set is:
\begin{equation}
P(a \neq y \mid i \in \mathcal{E}) = \frac{(1-\epsilon)(1-\rho)}{\epsilon\rho + (1-\epsilon)(1-\rho)}
\label{eq:enrichment}
\end{equation}
When $\epsilon = 0$, this equals 1: all misclassified samples are from minority groups.
\end{theorem}

\begin{proof}[Proof sketch]
Since $\ell_i^{(1)} = y_i$, we write $y$ for the true label. Under simplicity bias, the per-group error rates are $\epsilon$ on majority groups ($a = y$) and $1-\epsilon$ on minority groups ($a \neq y$). Applying Bayes' rule with prior $P(a \neq y) = 1-\rho$ yields \cref{eq:enrichment}. Full derivation in Appendix~\ref{app:thm_main_proof}.
\end{proof}

For \gls{waterbirds} ($\rho = 0.95$, $\epsilon = 0.02$), \cref{eq:enrichment} gives ${\sim}72\%$ minority share in the error set---a $14\times$ enrichment over their $5\%$ population fraction, which routes them into the hard branch. The bound extends to $C$-way classification by replacing $\rho$ with the per-class majority-attribute probability $\rho_y = P(a = a^{\star}(y) \mid y)$, preserving enrichment whenever $\rho_y > 1/C$.

\subsection{Benefit of Structural Separation}

We now provide motivation for why training separate classifiers for different subgroups can reduce worst-case risk, compared to training a single classifier on all data.

Within each class $c$, suppose the data naturally partitions into subgroups $\mathcal{S}_c^E$ (easy) and $\mathcal{S}_c^H$ (hard) that may require different features for optimal classification---corresponding to samples that will be assigned to nodes $2c$ (easy) and $2c{+}1$ (hard) after the first split.
Since each head performs one-vs-all classification, the head at node $2c$ distinguishes $\mathcal{S}_c^E$ from $\mathcal{S}_c^H\cup\mathcal{S}_{\neg c}$, while the head at node $2c{+}1$ distinguishes $\mathcal{S}_c^H$ from $\mathcal{S}_c^E\cup\mathcal{S}_{\neg c}$, where $\mathcal{S}_{\neg c}$ denotes all samples from the opposite class.

Let $\mathcal{G}$ be the hypothesis class of a classifier head, and let $R_{\mathcal{S}}(g)$ denote the risk of classifier $g$ on subgroup $\mathcal{S}$. For notational convenience, let $R_E(\cdot) = R_{\mathcal{S}_c^E}(\cdot)$ and $R_H(\cdot) = R_{\mathcal{S}_c^H}(\cdot)$. Define the optimal classifiers:
\begin{align}
g_E^{*} &= \arg\min_{g \in \mathcal{G}} R_E(g) \\
g_H^{*} &= \arg\min_{g \in \mathcal{G}} R_H(g)
\end{align}

\begin{definition}[Feature Conflict]
\label{def:conflict}
\emph{Feature conflict} occurs when the optimal hypotheses for the easy and hard subgroups differ: $g_E^{*} \neq g_H^{*}$.
\end{definition}

Feature conflict arises when the features minimizing risk on the majority group differ from those required for the minority group. While a single oracle classifier could in principle achieve low risk on both by relying on causal features, simplicity bias drives \gls{erm} toward the simpler majority solution $g_E^{*}$, conflicting with the minority-optimal $g_H^{*}$.

\begin{theorem}[Approximation Gap]
\label{thm:approx}
Under feature conflict, for any single classifier $g \in \mathcal{G}$:

\begin{equation}
\label{eq:gap}
\begin{split}
    \max\bigl(R_E(g), R_H(g)\bigr) &\geq \\
    \max\bigl(R_E(g_E^{*}), &R_H(g_H^{*})\bigr) + \Delta
\end{split}
\end{equation}

where $\Delta > 0$ when feature conflict exists.
\end{theorem}

\begin{proof}
The gap
\[
\Delta \;=\; \min_{g \in \mathcal{G}} \Bigl[\max\bigl(R_E(g), R_H(g)\bigr) - \max\bigl(R_E(g_E^{*}), R_H(g_H^{*})\bigr) \Bigr]
\]
is non-negative by construction. A strictly positive lower bound is established in Appendix~\ref{app:delta_bound} for a gaussian feature model with spurious magnitude $\mu_s$ and core magnitude $\mu_c$, taking the closed form
\[
\Delta \;=\; \Phi(-\mu_c) - \Phi\!\bigl(-\sqrt{\mu_s^2 + \mu_c^2}\bigr) \;>\; 0 \quad \text{whenever } \mu_s > 0,
\]
with $\Phi$ the standard normal CDF. The gap is monotone increasing in spurious strength: stronger spurious correlations make structural separation more beneficial. For a \gls{waterbirds}-like regime ($\mu_s = 2$, $\mu_c = 1$), $\Delta \approx 14.6\%$.
\end{proof}

\paragraph{Connection to NCT.} \Gls{nct} addresses the gap of Theorem~\ref{thm:approx} by training separate heads for each partition. By Theorem~\ref{thm:main}, under simplicity bias, the hard heads specialize on training data enriched with minority samples; at inference, the cross-head $\arg\max$ (\cref{eq:inference}) lets the head specialized on similar samples determine the prediction. The realized benefit depends on (i) routing fidelity (Theorem~\ref{thm:main}) and (ii) per-head learning quality. While the framework does not guarantee optimality, it provides a principled mechanism for specialization that parametric reweighting methods lack.

\section{Experimental Setup}

\textbf{Datasets.} We evaluate on five spurious-correlation benchmarks: \gls{waterbirds}~\citep{sagawa2019distributionally} (background), \gls{celeba}~\citep{liu2015faceattributes} (demographic), \gls{isic}~\citep{codella2018skin,tschandl2018ham10000} (acquisition artifacts), \gls{umnist}~\citep{sohoni2020no} (digit identity), and the 5-class \gls{cmnist}~\citep{arjovsky2019invariant,zhang2022correct} (color). Appendix~\ref{app:datasets} gives the full description.

\textbf{Evaluation Metrics.} Our primary metric is \gls{wga}---the lowest accuracy across pre-defined subgroups---reported alongside average accuracy. \gls{isic} uses \gls{auroc} following the GEORGE~\citep{sohoni2020no} evaluation protocol.

\textbf{Implementation Details.} Backbones are ImageNet-pretrained~\citep{deng2009imagenet} ResNet-50~\citep{he2016deep} for \gls{waterbirds}, \gls{celeba}, and \gls{isic}, and LeNet-5~\citep{lecun1998gradient} for \gls{umnist} and \gls{cmnist}. Models train for up to three iterations with AdamW~\citep{loshchilov2017decoupled}, using two phases per iteration (Section~\ref{sec:optimization}) and the depth criterion of Section~\ref{sec:depth_control}. Hyperparameters and per-dataset node-sampling strategies are tuned via Optuna~\citep{akiba2019optuna} against validation $\mathrm{pWGA}_2$; full configuration in Appendices~\ref{app:hyperparams} and~\ref{app:sampling}.

\textbf{Baselines.} We compare against eight baselines spanning three supervision tiers: \emph{group-supervised} (\gls{gdro}~\citep{sagawa2019distributionally}); \emph{validation-group-supervised} (\gls{jtt}~\citep{liu2021just}, \gls{cnc}~\citep{zhang2022correct}, \gls{eiil}~\citep{creager2021environment}, \gls{dfr}~\citep{kirichenko2022last}); and \emph{unsupervised} (\gls{erm}, GEORGE~\citep{sohoni2020no}, \gls{exmap}~\citep{chakraborty2024exmap}). Where published numbers are unavailable for \gls{isic} or \gls{umnist}, we adapt the official implementations under the same Optuna budget as \gls{nct}.

\section{Results and Analysis}

We evaluate \gls{nct} through three lenses: (1) robustness against spurious correlations compared to baselines, (2) quantitative interpretability through node-level sample alignment, and (3) qualitative interpretability through feature attribution. Auxiliary diagnostics --- pseudo-WGA proxy quality (Appendix~\ref{app:pwga_correlation}), cross-head calibration (Appendix~\ref{app:calibration}), training-time comparison (Appendix~\ref{app:training_time}), and component ablations (Appendix~\ref{app:component_ablations}) --- are deferred to the appendix.

\begin{table}[t]
\centering
\caption{\textbf{Main results.} \gls{wga} and average accuracy (\gls{auroc} for ISIC) across five benchmarks. Property columns: \textbf{T} = uses training group labels, \textbf{V} = uses validation group labels, \textbf{I} = inference-time interpretability of the discovered partition. Cells marked $^{*}$ are sourced from prior baseline papers. ``--'' indicates baselines not adapted to that benchmark.}
\label{tab:main_results}
\resizebox{\textwidth}{!}{%
\begin{tabular}{l ccc cc cc ccc cc cc}
\toprule
\multirow{2}{*}{\textbf{Method}}
 & \multicolumn{3}{c}{\textbf{Properties}}
 & \multicolumn{2}{c}{\textbf{Waterbirds}}
 & \multicolumn{2}{c}{\textbf{CelebA}}
 & \multicolumn{3}{c}{\textbf{ISIC (AUROC)}}
 & \multicolumn{2}{c}{\textbf{UMNIST}}
 & \multicolumn{2}{c}{\textbf{CMNIST}} \\
\cmidrule(lr){2-4}\cmidrule(lr){5-6}\cmidrule(lr){7-8}\cmidrule(lr){9-11}\cmidrule(lr){12-13}\cmidrule(lr){14-15}
 & T & V & I
 & WGA & Avg & WGA & Avg
 & Overall & Non-Patch & Histopath.
 & WGA & Avg & WGA & Avg \\
\midrule
Group DRO & \ding{51} & \ding{51} & \ding{55}
 & 90.7$^{*}$\,\tiny{$\pm$0.4} & 92.7$^{*}$\,\tiny{$\pm$0.4}
 & 89.3$^{*}$\,\tiny{$\pm$0.9} & 92.8$^{*}$\,\tiny{$\pm$0.1}
 & .933$^{*}$\,\tiny{$\pm$.005} & .923$^{*}$\,\tiny{$\pm$.003} & .875$^{*}$\,\tiny{$\pm$.004}
 & 96.8$^{*}$\,\tiny{$\pm$0.4} & 98.0$^{*}$\,\tiny{$\pm$0.3}
 & 78.5$^{*}$\,\tiny{$\pm$4.5} & 90.6$^{*}$\,\tiny{$\pm$0.1} \\
\midrule
JTT  & \ding{55} & \ding{51} & \ding{55}
     & 86.7$^{*}$ & 93.3$^{*}$ & 81.1$^{*}$ & 88.0$^{*}$
     & .892\,\tiny{$\pm$.005} & .862\,\tiny{$\pm$.006} & .827\,\tiny{$\pm$.008}
     & 92.2\,\tiny{$\pm$1.0} & 98.5\,\tiny{$\pm$0.2}
     & 74.5$^{*}$\,\tiny{$\pm$2.4} & 90.2$^{*}$\,\tiny{$\pm$0.8} \\
CnC  & \ding{55} & \ding{51} & \ding{55}
     & 88.5$^{*}$\,\tiny{$\pm$0.3} & 90.9$^{*}$\,\tiny{$\pm$0.1}
     & 88.8$^{*}$\,\tiny{$\pm$0.9} & 89.9$^{*}$\,\tiny{$\pm$0.5}
     & .951\,\tiny{$\pm$.008} & .915\,\tiny{$\pm$.014} & .870\,\tiny{$\pm$.025} & 92.8\,\tiny{$\pm$1.8} & 97.3\,\tiny{$\pm$0.8}
     & 77.4$^{*}$\,\tiny{$\pm$3.0} & 90.9$^{*}$\,\tiny{$\pm$0.6} \\
EIIL & \ding{55} & \ding{51} & \ding{55}
     & 87.3$^{*}$ & 93.1$^{*}$ & 81.3$^{*}$ & 89.5$^{*}$
     & .958\,\tiny{$\pm$.007} & .921\,\tiny{$\pm$.013} & .877\,\tiny{$\pm$.017}
     & 94.0\,\tiny{$\pm$1.5} & 97.6\,\tiny{$\pm$0.2}
     & 72.8$^{*}$\,\tiny{$\pm$6.8} & 90.7$^{*}$\,\tiny{$\pm$0.9} \\
DFR  & \ding{55} & \ding{51} & \ding{55}
     & 92.1$^{*}$ & 96.7$^{*}$ & 86.9$^{*}$ & 91.1$^{*}$
     & .957\,\tiny{$\pm$.006} & .922\,\tiny{$\pm$.008} & .876\,\tiny{$\pm$.011}
     & 97.1\,\tiny{$\pm$0.1} & 97.8\,\tiny{$\pm$0.0}
     & -- & -- \\
\midrule
ERM    & \ding{55} & \ding{55} & \ding{55}
       & 63.3$^{*}$\,\tiny{$\pm$1.6} & 97.3$^{*}$\,\tiny{$\pm$0.1}
       & 40.3$^{*}$\,\tiny{$\pm$2.3} & 95.7$^{*}$\,\tiny{$\pm$0.0}
       & .957$^{*}$\,\tiny{$\pm$.002} & .922$^{*}$\,\tiny{$\pm$.003} & .875$^{*}$\,\tiny{$\pm$.005}
       & 93.9$^{*}$\,\tiny{$\pm$0.6} & 98.7$^{*}$\,\tiny{$\pm$0.1}
       & 0.0$^{*}$\,\tiny{$\pm$0.0} & 20.1$^{*}$\,\tiny{$\pm$0.2} \\
GEORGE & \ding{55} & \ding{55} & \ding{55}
       & 76.2$^{*}$\,\tiny{$\pm$2.0} & 95.7$^{*}$\,\tiny{$\pm$0.5}
       & 53.7$^{*}$\,\tiny{$\pm$1.3} & 94.6$^{*}$\,\tiny{$\pm$0.2}
       & .927$^{*}$\,\tiny{$\pm$.008} & .912$^{*}$\,\tiny{$\pm$.005} & .876$^{*}$\,\tiny{$\pm$.006}
       & 95.7$^{*}$\,\tiny{$\pm$0.6} & 98.1$^{*}$\,\tiny{$\pm$0.3}
       & 76.4$^{*}$\,\tiny{$\pm$2.3} & 89.5$^{*}$\,\tiny{$\pm$0.3} \\
ExMap  & \ding{55} & \ding{55} & \ding{55}
       & 92.5$^{*}$ & 96.0$^{*}$ & 84.4$^{*}$ & 91.8$^{*}$
       & .957\,\tiny{$\pm$.004} & .923\,\tiny{$\pm$.008} & .878\,\tiny{$\pm$.017}
       & 96.7\,\tiny{$\pm$0.3} & 97.6\,\tiny{$\pm$0.1}
       & -- & -- \\
\rowcolor{gray!12}
NCT (Ours) & \ding{55} & \ding{55} & \ding{51}
       & 88.0\,\tiny{$\pm$0.9} & 92.8\,\tiny{$\pm$0.7}
       & 86.1\,\tiny{$\pm$0.9} & 88.4\,\tiny{$\pm$0.9}
       & .959\,\tiny{$\pm$.001} & .924\,\tiny{$\pm$.002} & .880\,\tiny{$\pm$.002}
       & 93.7\,\tiny{$\pm$2.5} & 97.5\,\tiny{$\pm$0.4}
       & 72.9\,\tiny{$\pm$2.8} & 88.4\,\tiny{$\pm$1.5} \\
\bottomrule
\end{tabular}%
}
\end{table}

\subsection{Quantitative Performance: Mitigating Spurious Correlations}

Table~\ref{tab:main_results} reports the full comparison (per-iteration breakdowns and per-seed depth selection are in Appendix~\ref{app:depth_diagnostics}). On \gls{isic}, all three of \gls{nct}'s \gls{auroc} numbers lead the table, narrowly above Group DRO despite using no training group labels. Within the unsupervised tier, \gls{nct} improves on \gls{erm} and GEORGE on almost every benchmark and is competitive with \gls{exmap} on \gls{waterbirds} and \gls{celeba}. Against validation-supervised methods, \gls{nct} matches or narrowly exceeds them on \gls{waterbirds}, is comparable on \gls{cmnist} and \gls{celeba}, but trails on \gls{umnist}. Overall, \gls{nct} is competitive with the strongest baselines on most benchmarks in unsupervised tier while remaining the only unsupervised entry that exposes the discovered partition structurally.

\subsection{Structural Interpretability: Latent Group Discovery}
\label{sec:structural_interpretability}

\begin{figure}[t]
    \centering
    \includegraphics[width=\textwidth]{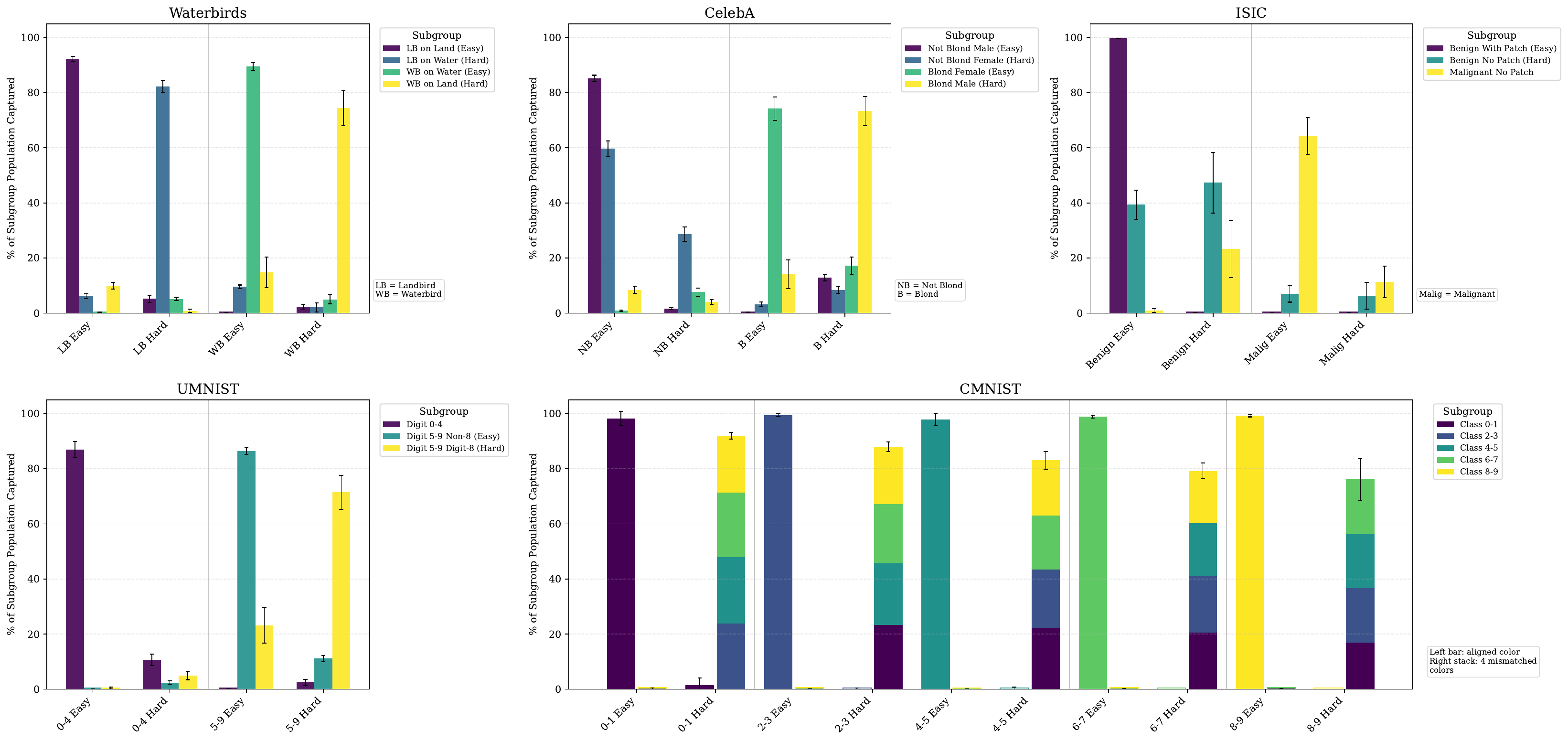}
    \caption{\textbf{Iteration-2 capture rates.} Bar height = \% of a subgroup's population routed to each leaf (mean $\pm$ std, five seeds). For \gls{cmnist}, each leaf shows the own-class matching-colour subgroup (left bar) and the four own-class mismatched-colour subgroups stacked (right bar).}
    \label{fig:node_dynamics}
\end{figure}

Unlike robust optimization methods that implicitly reweight samples, \gls{nct} generates an explicit partition of the data. We quantify this via the iteration-2 capture rate: the percentage of a ground-truth subgroup's population routed to each leaf, averaged across five seeds. Across all datasets, majority subgroups concentrate in easy leaves while minority subgroups concentrate in hard leaves (Figure~\ref{fig:node_dynamics}).

\textbf{Waterbirds.} The hard \emph{landbird} leaf captures $82.2\%$ of the minority \textit{landbird-on-water}, while the easy leaf retains $92.3\%$ of the majority \textit{landbird-on-land}. The waterbird branch mirrors this: hard captures $74.4\%$ of \textit{waterbird-on-land}, easy retains $89.5\%$ of \textit{waterbird-on-water}.

\textbf{CelebA.} The \textit{blond-male} concentrates in its hard leaf ($73.3\%$), and the not-blond hard leaf catches $28.7\%$ of \textit{not-blond-female}. Easy leaves retain $85.1\%$ of \textit{not-blond-male} and $74.2\%$ of \textit{blond-female}.

\textbf{ISIC.} The easy benign leaf captures $99.7\%$ of \textit{benign-with-patch}, exploiting the color-patch shortcut. The hard benign leaf isolates $47.3\%$ of the \textit{benign-no-patch} subgroup, forcing reliance on lesion-based features rather than the artifact.

\textbf{UMNIST.} The hard $5$--$9$ leaf captures $71.4\%$ of the undersampled \textit{digit 8}, separating it from the majority \textit{digit 5--9 (non-8)} population which the easy leaf retains at $86.4\%$.

\textbf{CMNIST.} The same pattern carries to the multi-class case. Within each class, samples whose color matches the class consolidate in its easy leaf ($98.7\%$), while samples whose color does not match concentrate in the hard leaf of the true class ($83.7\%$). The four off-class colors contribute roughly equally to each hard-leaf stack (Figure~\ref{fig:node_dynamics}), so the routing isolates the color-mismatched samples regardless of which off-class color they carry.

For datasets with only three inherent subgroups (\gls{isic}, \gls{umnist}), iteration 2 produces four leaves but one remains sparsely populated rather than artificially splitting a natural subgroup; the sparse leaf absorbs residual samples from the majority without disrupting the primary split.

We further analyze routing in Iteration~3 (eight leaves) in Appendix~\ref{app:node_analysis}. The hierarchical pattern persists at finer granularity: minority subgroups continue to concentrate in their class's hard branch (\gls{waterbirds} 75.1\% landbird-on-water, \gls{celeba} 60.3\% blond-male, \gls{umnist} 44.7\% digit-8), and on \gls{waterbirds} the sparse-node merging mechanism (Section~\ref{sec:depth_control}) activates to drop unused leaves.

\subsection{Qualitative Interpretability: Attribution Analysis}

\begin{figure}[t]
    \centering
    \includegraphics[width=1\textwidth]{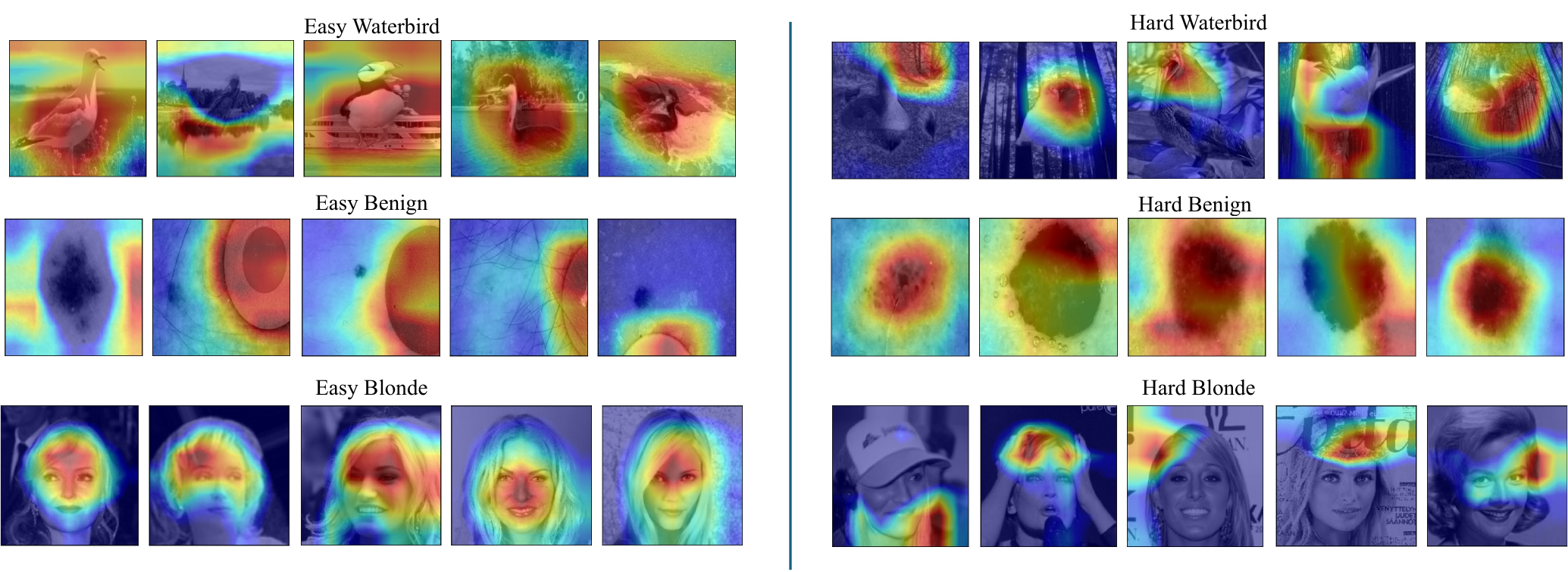}
    \caption{\textbf{Iteration-2 LayerGradCAM attributions.} Easy leaves (left) localize on the spurious cue---background on Waterbirds, the color-patch artifact on ISIC, the face on CelebA. Hard leaves (right) shift attention to the semantic class feature: the bird body, the lesion, and the hair respectively.}
    \label{fig:attribution}
\end{figure}

The capture-rate analysis shows that hard leaves isolate minority subgroups, but do they actually rely on different features? We use LayerGradCAM~\citep{selvaraju2017grad} to inspect what image regions each leaf attends to; if the structural separation is meaningful, easy leaves should rely on the spurious cue while hard leaves should attend to the semantic class feature. Figure~\ref{fig:attribution} shows five randomly selected samples per leaf on \gls{waterbirds}, \gls{isic}, and \gls{celeba}.

\textbf{Waterbirds.} The easy waterbird leaf attends to the water and surrounding scene---the lake surface and the wake behind the bird---rather than the bird itself. The hard leaf localizes tightly on the bird's body across both water and forest backgrounds.

\textbf{ISIC.} The easy benign leaf locks onto the colored skin-marker ring that frames many benign lesions in training, while the central lesion receives little attention. The hard benign leaf inverts this: it centers on the lesion---its boundary and pigmented body---and ignores the ring artifact.

\textbf{CelebA.} The easy blond leaf attends to the lower face---eyes, cheeks, mouth---using it as a gender proxy. The hard blond leaf shifts upward to the hair, the actual class signal; in the samples shown, this includes a cap occluding the hair and faces where the gender shortcut would mislead.

We omit \gls{umnist} and \gls{cmnist} from this figure: their spurious cues are global properties like digit identity/color rather than localizable regions, so attribution maps cannot meaningfully highlight them.

\section{Conclusion}
\label{sec:conclusion}

Existing methods for spurious correlations adjust parameters to improve worst-group accuracy but leave the classifier opaque about the latent groups it has learned. \Gls{nct} takes a different approach: training difficulty becomes a routing signal that partitions samples into easy and hard branches over successive iterations, with specialized heads at every node. The resulting tree serves as both classifier and partition---each leaf encodes a predicted class together with the difficulty path that produced it, and hard branches consistently isolate minority subgroups across binary and multi-class spurious correlations. Across five benchmarks, this framework delivers worst-group accuracy competitive with strong baselines---leading on \gls{isic} across all supervision tiers---while making the latent group structure visible at inference rather than hidden in the parameters.

The structural guarantees come with two conditions. First, routing relies on simplicity bias (Theorem~\ref{thm:main}): if minority features are themselves easy to learn, difficulty-based routing fails to isolate them. Second, the depth-selection rule depends on $\mathrm{pWGA}_2$ tracking true \gls{wga}, and may stop early when the proxy collapses faster than the true metric. Severe class imbalance can potentially amplify both, since a small misrouting rate can let an oversized easy class outnumber the genuine minority in a hard branch.

We identify two natural directions for future work. First, extending \gls{nct} to language settings---where spurious correlations arise from lexical artifacts or demographic markers---would test whether simplicity bias provides the same routing signal under different feature geometry. Second, the routing rule can admit richer signals beyond correctness, opening a natural research direction: gradient-based, attribution-based, or representation-disagreement criteria could partition data along axes that correctness alone may miss. We see \gls{nct} as an evidence that the latent group structure of a dataset can be recovered as the architecture itself, discovered without supervision and visible at inference.

\section*{Acknowledgement}
We acknowledge the support and funding of CIFAR, Natural Sciences and Engineering Research Council of Canada (NSERC) and the Canada Foundation for Innovation (CFI). This research was enabled through computational resources provided by the Research Computing Services group at the University of Calgary and the Digital Research Alliance of Canada.

\bibliographystyle{unsrtnat}
\bibliography{main}



\newpage
\section*{Appendix Contents}
\startcontents[appendices]
\printcontents[appendices]{}{1}{\setcounter{tocdepth}{2}}

\newpage

\appendix

\section{Extended Experimental Setup}
\label{app:setup}

\subsection{Dataset Details}
\label{app:datasets}

We evaluate on five benchmark datasets exhibiting spurious correlations. Table~\ref{tab:dataset_summary} summarizes the dataset characteristics and split sizes, and Table~\ref{tab:group_distributions} provides the training set group distributions.

\begin{table}[h]
\centering
\caption{\textbf{Dataset Summary.} Benchmark datasets with task descriptions, spurious attributes, and split sizes.}
\label{tab:dataset_summary}
\begin{tabular}{lllrrr}
\toprule
\textbf{Dataset} & \textbf{Task} & \textbf{Spurious Attribute} & \textbf{Train} & \textbf{Val} & \textbf{Test} \\
\midrule
\gls{waterbirds} & Bird type (binary)         & Background       & 4,795   & 1,199  & 5,794 \\
\gls{celeba}     & Hair color (binary)        & Gender           & 162,770 & 19,867 & 19,962 \\
\gls{isic}       & Lesion diagnosis (binary)  & Colored patches  & 19,124  & 2,390  & 2,392 \\
\gls{umnist}     & Digit range (binary)       & Digit identity   & 43,542  & 12,000 & 10,000 \\
\gls{cmnist}     & Digit pair (5-class)       & Color            & 54,000  & 6,000  & 10,000 \\
\bottomrule
\end{tabular}
\end{table}

\paragraph{Waterbirds.} Constructed by superimposing bird images from CUB-200-2011~\citep{wah2011caltech} onto backgrounds from Places~\citep{zhou2017places}. The task is binary classification of bird type (waterbird vs.\ landbird), where background (water vs.\ land) is spuriously correlated with the label. In training, 95\% of waterbirds appear on water and 95\% of landbirds on land. The validation and test sets are balanced across groups.

\paragraph{CelebA.} Binary hair color classification (blond vs.\ non-blond) on the \gls{celeba} dataset~\citep{liu2015faceattributes}, where gender serves as the spurious attribute. Blond males constitute only 1,387 samples (0.85\%) in training, making this the critical minority group.

\paragraph{ISIC.} Dermoscopic images from \gls{isic} 2019~\citep{tschandl2018ham10000,codella2018skin} for binary classification (benign vs.\ malignant). Following~\citet{sohoni2020no}, colored patches (acquisition artifacts) that appear predominantly on benign lesions serve as the spurious attribute.

\paragraph{UMNIST.} A modified MNIST~\citep{sohoni2020no} with binary classification: digits 0--4 vs.\ 5--9. Digit `8' is undersampled to 5\% of its original frequency in training, creating a spurious association between digit identity and class membership.

\paragraph{CMNIST.} A 5-class colored MNIST variant following~\citet{zhang2022correct}: digits are grouped into pairs $\{0,1\},\{2,3\},\{4,5\},\{6,7\},\{8,9\}$ and each class is dyed predominantly with a fixed color. The training set has 99.5\% color--class correlation; the validation and test sets are uncorrelated. The full $25 = 5 \times 5$ class-by-color grid serves as the ground-truth group structure, giving a multi-class spurious-correlation setting absent from the other four benchmarks. Aligned subgroups average $\approx 10{,}746$ samples each; conflicting subgroups average $\approx 13$ samples each.

\begin{table}[h]
\centering
\caption{\textbf{Group Distributions.} Training set composition by subgroup. Minority groups are italicized. For \gls{cmnist}, the 5 aligned and 20 conflicting subgroups are aggregated for compactness.}
\label{tab:group_distributions}
\begin{tabular}{llrr}
\toprule
\textbf{Dataset} & \textbf{Group} & \textbf{Count} & \textbf{\%} \\
\midrule
\multirow{4}{*}{\gls{waterbirds}}
    & Landbird on Land & 3,498 & 72.9 \\
    & \textit{Landbird on Water} & \textit{184} & \textit{3.8} \\
    & \textit{Waterbird on Land} & \textit{56} & \textit{1.2} \\
    & Waterbird on Water & 1,057 & 22.0 \\
\midrule
\multirow{4}{*}{\gls{celeba}}
    & Not Blond Female & 71,629 & 44.0 \\
    & Not Blond Male & 66,874 & 41.1 \\
    & Blond Female & 22,880 & 14.1 \\
    & \textit{Blond Male} & \textit{1,387} & \textit{0.9} \\
\midrule
\multirow{3}{*}{\gls{isic}}
    & \textit{Benign No Patch} & \textit{9,861} & \textit{51.6} \\
    & Benign With Patch & 7,420 & 38.8 \\
    & Malignant & 1,843 & 9.6 \\
\midrule
\multirow{3}{*}{\gls{umnist}}
    & Digits 0--4 & 24,449 & 56.2 \\
    & Digits 5--9 (non-8) & 18,859 & 43.3 \\
    & \textit{Digit 8} & \textit{234} & \textit{0.5} \\
\midrule
\multirow{2}{*}{\gls{cmnist}}
    & Class--color aligned (5 groups) & 53,730 & 99.5 \\
    & \textit{Class--color conflicting (20 groups)} & \textit{270} & \textit{0.5} \\
\bottomrule
\end{tabular}
\end{table}

\subsection{Preprocessing}
\label{app:preprocessing}
Per-dataset preprocessing is summarized in Table~\ref{tab:preprocessing}. We adopt the published preprocessing pipelines of the corresponding baselines verbatim and do not alter resolution, augmentation, or normalization. \gls{waterbirds} follows~\citet{sagawa2019distributionally}; \gls{celeba} follows the ImageNet-generic recipe used by~\citet{kirichenko2022last}; \gls{isic} and \gls{umnist} follow~\citet{sohoni2020no}; \gls{cmnist} follows~\citet{zhang2022correct}.

\begin{table}[h]
\centering
\caption{\textbf{Preprocessing pipelines.} All settings are taken from the cited baselines without modification. ``RRC'' = \texttt{RandomResizedCrop}; ``RC$(s, p)$'' = \texttt{RandomCrop} of size $s$ with padding $p$; ``CC'' = \texttt{CenterCrop}; ``HFlip''/``VFlip'' = horizontal/vertical flip with $p=0.5$.}
\label{tab:preprocessing}
\resizebox{\textwidth}{!}{%
\begin{tabular}{l c l l l}
\toprule
\textbf{Dataset} & \textbf{Res.} & \textbf{Train transform} & \textbf{Eval transform} & \textbf{Normalization} \\
\midrule
\gls{waterbirds} & $224$ & RRC$(224)$, scale $[0.7,1.0]$, ratio $[0.75,1.33]$; HFlip & Resize $256 \to$ CC $224$ & ImageNet \\
\gls{celeba}     & $224$ & RRC$(224)$, scale $[0.7,1.0]$, ratio $[0.75,1.33]$; HFlip & Resize $256 \to$ CC $224$ & ImageNet \\
\gls{isic}       & $224$ & RRC$(224)$, scale $[0.7,1.0]$, ratio $[0.75,1.33]$; HFlip, VFlip & Resize $256 \to$ CC $224$ & ISIC-specific \\
\gls{umnist}     & $32$  & RC$(28, 4) \to$ Resize $32$; HFlip & Resize $32$ & MNIST $(0.131, 0.308)$ \\
\gls{cmnist}     & $32$  & Resize $40 \to$ RC$(32, 0)$; no flip & Resize $40 \to$ CC $32$ & $(0.5, 0.5)$ \\
\bottomrule
\end{tabular}%
}
\end{table}

\section{Implementation Details}
\label{app:implementation}

\subsection{Architecture}
\label{app:architecture}

For \gls{waterbirds}, \gls{celeba}, and \gls{isic}, we use a ResNet-50~\citep{he2016deep} backbone initialized with ImageNet~\citep{deng2009imagenet} pretrained weights. For \gls{umnist} we use a LeNet-4~\citep{lecun1998gradient} backbone, and for \gls{cmnist} a LeNet-5 backbone, both trained from scratch. At iteration~1 each class head is an MLP block (\gls{waterbirds}, \gls{celeba}, \gls{isic}) or a Linear classifier (\gls{umnist}, \gls{cmnist}). At iterations~2--3 every dataset uses the asymmetric-head configuration of Section~\ref{sec:architecture}: easy children are linear and hard children use a Linear$\rightarrow$ReLU$\rightarrow$Dropout block. The hidden block is wrapped by an input batch-norm and a hidden batch-norm; iteration-1 heads omit these batch-norm layers (they are replaced by identity), so the additional normalization only takes effect once the hierarchy starts to grow.

\subsection{Training Protocol}
\label{app:training_protocol}

Each iteration consists of two phases (Section~\ref{sec:optimization}); the per-iteration Phase~1 fraction is dataset-specific (Table~\ref{tab:hyperparams}). Iteration~1 trains for only 1--3 epochs --- its purpose is hard-sample identification, not convergence, mirroring the JTT-style identification stage~\citep{liu2021just}. Iterations 2+ train for 50--100 epochs with early stopping. Checkpoint selection uses validation loss at iteration~1 (the routing partition does not yet exist) and $\mathrm{pWGA}_2$ thereafter.

\subsection{Hyperparameters}
\label{app:hyperparams}

Table~\ref{tab:hyperparams} lists the dataset-specific hyperparameters we tuned per benchmark. Choices held constant across all five datasets are: one-vs-all BCE loss, asymmetric heads from iteration~2 onward (linear easy child, linear--ReLU--dropout hard child), sparse-merge threshold $m_{\min}=20$, and depth-selection tolerance $z_\alpha = 1.96$. The optimizer is AdamW for every dataset except \gls{celeba}, where SGD with momentum~$0.9$ is used.

\begin{table}[h]
\centering
\caption{\textbf{Dataset-specific hyperparameters.} Constants held across all datasets (loss, asymmetric head structure, $m_{\min}$, $z_\alpha$) are listed in the surrounding text; backbone and head architecture choices are described in Appendix~\ref{app:architecture}.}
\label{tab:hyperparams}
\small
\setlength{\tabcolsep}{4pt}
\begin{tabular}{lccccc}
\toprule
\textbf{Parameter} & \textbf{\gls{waterbirds}} & \textbf{\gls{celeba}} & \textbf{\gls{isic}} & \textbf{\gls{umnist}} & \textbf{\gls{cmnist}} \\
\midrule
Head hidden dim & 32 & 128 & 96 & 8 & 64 \\
Head dropout & 0.3 & 0.3 & 0 & 0 & 0.2 \\
\midrule
Backbone LR & $2.7 \times 10^{-5}$ & $1 \times 10^{-5}$ & $8 \times 10^{-5}$ & $5 \times 10^{-3}$ & $2 \times 10^{-3}$ \\
Head LR & $1.3 \times 10^{-5}$ & $1 \times 10^{-4}$ & $1 \times 10^{-4}$ & $5 \times 10^{-4}$ & $1 \times 10^{-5}$ \\
LR decay factor & 2.5 & 1.5 & 2.5 & 1.5 & 1.0 \\
Weight decay & $2.4 \times 10^{-3}$ & $0.1$ & $0.1$ & $0.1$ & $5 \times 10^{-5}$ \\
\midrule
Batch size & 128 & 64 & 256 & 64 & 64 \\
Epochs (iter 1/2/3) & 1/100/100 & 1/50/50 & 1/50/50 & 3/50/50 & 2/50/50 \\
Phase 1 ratio (iter 1/2/3) & 0.0/0.3/0.3 & 1.0/0.3/0.3 & 1.0/0.5/0.5 & 0.0/0.5/0.5 & 0.0/0.2/0.7 \\
Scheduler & Plateau & Plateau & Step & Plateau & Plateau \\
Early stop patience & 15 & 10 & 15 & 15 & 5 \\
\midrule
Sampling strategy & Class weights & Downsample & Class weights & Geomean & Geomean \\
Aux loss weight ($\lambda_{\text{aux}}$) & 1.0 & 2.0 & 1.0 & 0.5 & 1.0 \\
Class weight cap & 40 & 40 & 40 & 40 & 40 \\
\bottomrule
\end{tabular}
\end{table}

\subsection{Hyperparameter Search}
\label{app:hp_search}

Hyperparameter values in Table~\ref{tab:hyperparams} were selected with Optuna~\citep{akiba2019optuna} using the Tree-structured Parzen Estimator (TPE) sampler over 50--100 trials per dataset, with 20 warm-up trials of random sampling before switching to TPE. The optimization objective was $\mathrm{pWGA}_2$ on the validation set (Section~\ref{sec:depth_control}), computed using routing-derived pseudo-labels rather than ground-truth subgroup annotations. The search ranged over backbone and head learning rates ($10^{-6}$ to $10^{-2}$), LR decay factor (1.0 to 3.0), weight decay ($10^{-5}$ to 2.0), head hidden dimension (16 to 128), head dropout (0.0 to 0.5), Phase~1 training ratios, auxiliary loss weight (0.0 to 2.0) and sampling strategy (\{class weights, downsample, geomean\}).

\section{Baseline Implementations}
\label{app:baselines}

\paragraph{Group DRO~\citep{sagawa2019distributionally}.} All Group DRO numbers in Table~\ref{tab:main_results} are taken directly from prior baseline papers: \gls{waterbirds} and \gls{celeba} from~\citet{sagawa2019distributionally}, \gls{isic} and \gls{umnist} from~\citet{sohoni2020no}, and \gls{cmnist} from~\citet{zhang2022correct}. We do not retrain Group DRO ourselves on any dataset.

\paragraph{GEORGE~\citep{sohoni2020no}.} GEORGE clusters the ERM feature space into pseudo-subgroups, trains Group DRO on those clusters, and selects the best checkpoint by worst-cluster validation accuracy---no ground-truth group labels are accessed at any stage. We do not retrain GEORGE ourselves: \gls{waterbirds}, \gls{celeba}, \gls{isic}, and \gls{umnist} numbers are taken from~\citet{sohoni2020no} and \gls{cmnist} numbers from~\citet{zhang2022correct}.

\paragraph{JTT~\citep{liu2021just}.} \gls{waterbirds} and \gls{celeba} numbers are taken from~\citet{liu2021just}; \gls{cmnist} numbers are taken from~\citet{zhang2022correct}. For \gls{isic} and \gls{umnist} we extend the official JTT codebase. The two JTT-specific knobs---identification epochs $T_\text{up}$ and upweight factor $\lambda_\text{up}$---are tuned via Optuna against validation \gls{wga}, yielding $(T_\text{up}, \lambda_\text{up}) = (1, 50)$ on \gls{isic} and $(1, 20)$ on \gls{umnist}. Model selection uses validation worst-group accuracy.

\paragraph{CnC~\citep{zhang2022correct}.} \gls{waterbirds}, \gls{celeba}, and \gls{cmnist} numbers are taken from~\citet{zhang2022correct}. For \gls{isic} and \gls{umnist} we use the official release and tune the stage-2 contrastive parameters---number of anchors / positives / negatives / easy-negatives, contrastive weight, temperature, and contrastive batch factor---via Optuna. The tuned configuration converges to $17$ samples per role, contrastive weight $0.75$, temperature $0.1$, and batch factor $32$.

\paragraph{DFR~\citep{kirichenko2022last}.} \gls{waterbirds} and \gls{celeba} numbers are taken from~\citet{kirichenko2022last}. For \gls{isic} and \gls{umnist} we use the paper's main variant ($\text{DFR}^{\text{Val}}_{\text{Tr}}$): the ERM-trained backbone is frozen and a logistic head is retrained on a group-balanced validation subset, with $20$ retrains at evaluation. The inverse regularisation strength $C$ is tuned via Optuna; class weights are not tuned, following the paper's main-variant recipe. We do not report \gls{cmnist}: the released variant is binary, and the multi-class adaptation requires non-trivial changes to the group-balanced retraining objective that we leave to future work.

\paragraph{EIIL~\citep{creager2021environment}.} \gls{waterbirds} and \gls{celeba} numbers are taken from~\citet{creager2021environment}; \gls{cmnist} numbers are taken from~\citet{zhang2022correct}. For \gls{isic} and \gls{umnist} we use the official two-stage pipeline. Stage-1 environment-inference uses the reference defaults ($2$ environments, environment lr $10^{-3}$, $10\,000$ inference steps). Stage-2 trains a Group DRO learner on the inferred environments ($\eta = 0.01$, patience $10$); the remaining stage-2 hyperparameters are tuned via Optuna against validation worst-group accuracy.

\paragraph{ExMap~\citep{chakraborty2024exmap}.} \gls{waterbirds} and \gls{celeba} numbers are taken from~\citet{chakraborty2024exmap}. For \gls{isic} and \gls{umnist} we use the global G-ExMap variant of the released pipeline: spectral clustering of \gls{lrp} attribution maps with the cluster count chosen automatically by the eigengap heuristic and \gls{lrp} $\gamma = 4$, followed by the authors' DFR-style last-layer retraining with $20$ retrains at evaluation. The LRP per-batch size is tuned via Optuna and converges to $8$ on \gls{isic} and $512$ on \gls{umnist}. We do not report \gls{cmnist}: ExMap inherits DFR's binary group-balancing for the retraining stage, and a faithful multi-class extension is non-trivial and left to future work.

\section{Cross-Head Calibration Diagnostics}
\label{app:calibration}

Because every \gls{nct} head is trained with independent \gls{bce}, the $\arg\max$ inference rule (\cref{eq:inference}) implicitly assumes that head logits are on comparable scales. Three implicit factors push them in that direction: the uniform binary BCE target, the shared parent-layer representation, and the per-iteration batch-norm modules introduced in Appendix~\ref{app:architecture}. We audit the resulting calibration with two diagnostics: the per-head logit mean/std across all test samples, and the decision margin (winning logit minus second-best logit) at each predicted leaf.

\paragraph{Per-head logit statistics.} Table~\ref{tab:logit_stats} reports the mean and standard deviation of each head's logit distribution on the test set, averaged across the five seeds. At iteration~1 the two class-head means are near-antisymmetric (a consequence of \gls{bce} with flipped binary targets). At iteration~2 the means diverge, but in every case the per-head standard deviations are comparable to or larger than the typical mean separation, so the logit distributions overlap and argmax can still route samples to the head that fits them best.

\begin{table}[h]
\centering
\caption{\textbf{Per-head logit mean $\pm$ standard deviation} on the test set, averaged across five seeds. Heads that were merged by the sparse-node rule and never selected are omitted. \gls{cmnist} is summarized in prose because its 5 iter-1 heads and 10 iter-2 heads do not fit the column layout.}
\label{tab:logit_stats}
\resizebox{\textwidth}{!}{%
\begin{tabular}{l cc cccc}
\toprule
& \multicolumn{2}{c}{\textbf{Iter 1}} & \multicolumn{4}{c}{\textbf{Iter 2}} \\
\cmidrule(lr){2-3} \cmidrule(lr){4-7}
Dataset & H0 & H1 & H0 & H1 & H2 & H3 \\
\midrule
\gls{waterbirds} & $+0.27$\,\tiny{$\pm$1.01} & $-0.16$\,\tiny{$\pm$0.90} & $-1.09$\,\tiny{$\pm$2.06} & $+0.46$\,\tiny{$\pm$1.37} & $-0.70$\,\tiny{$\pm$1.55} & $-0.39$\,\tiny{$\pm$1.13} \\
\gls{celeba}     & $+0.39$\,\tiny{$\pm$0.60} & $-0.39$\,\tiny{$\pm$0.62} & $+1.17$\,\tiny{$\pm$2.02} & $-0.39$\,\tiny{$\pm$0.94} & $-1.27$\,\tiny{$\pm$1.39} & $-0.65$\,\tiny{$\pm$0.81} \\
\gls{isic}       & $+0.80$\,\tiny{$\pm$1.44} & $-0.81$\,\tiny{$\pm$1.42} & $+0.47$\,\tiny{$\pm$5.82} & $-2.75$\,\tiny{$\pm$2.99} & $-7.26$\,\tiny{$\pm$4.64} & $-5.01$\,\tiny{$\pm$1.96} \\
\gls{umnist}     & $+0.89$\,\tiny{$\pm$6.21} & $-0.91$\,\tiny{$\pm$6.22} & $-2.50$\,\tiny{$\pm$6.24} & $-3.81$\,\tiny{$\pm$2.83} & $-4.57$\,\tiny{$\pm$6.29} & $-4.45$\,\tiny{$\pm$3.99} \\
\bottomrule
\end{tabular}%
}
\end{table}

For \gls{cmnist} (5-class, 5 heads at iter~1 and 10 at iter~2), the iter-1 head means span $[-5.49,\,-4.05]$ with an average within-head std of $5.59$, and the iter-2 head means span $[-1.21,\,+0.63]$ with an average within-head std of $4.00$. The per-head std is several times larger than the spread of means at every iteration, so the same overlap argument applies.

\paragraph{Decision margins.} Table~\ref{tab:margins} reports the decision margin (winning logit minus runner-up) per predicted leaf, averaged across five seeds. All margins are strictly positive, confirming that argmax is decisive even when per-head means differ. The smallest binary-task margin is $0.82$ (\gls{waterbirds} iter-2 H0); for \gls{cmnist} the minimum mean margin across the 5 iter-1 heads is $10.95$, and the minimum across the 10 iter-2 heads is $7.71$.

\begin{table}[h]
\centering
\caption{\textbf{Mean decision margin $\pm$ std} per predicted leaf (averaged across five seeds). \gls{cmnist} margins are summarized in prose for the same column-layout reason.}
\label{tab:margins}
\resizebox{\textwidth}{!}{%
\begin{tabular}{l cc cccc}
\toprule
& \multicolumn{2}{c}{\textbf{Iter 1}} & \multicolumn{4}{c}{\textbf{Iter 2}} \\
\cmidrule(lr){2-3} \cmidrule(lr){4-7}
Dataset & H0 & H1 & H0 & H1 & H2 & H3 \\
\midrule
\gls{waterbirds} & $1.94$\,\tiny{$\pm$0.84} & $1.46$\,\tiny{$\pm$0.85} & $0.82$\,\tiny{$\pm$0.26} & $3.26$\,\tiny{$\pm$1.71} & $3.04$\,\tiny{$\pm$1.37} & $2.39$\,\tiny{$\pm$1.79} \\
\gls{celeba}     & $1.33$\,\tiny{$\pm$0.71} & $0.99$\,\tiny{$\pm$0.72} & $2.81$\,\tiny{$\pm$1.79} & $1.34$\,\tiny{$\pm$0.97} & $2.52$\,\tiny{$\pm$1.81} & $0.90$\,\tiny{$\pm$0.66} \\
\gls{isic}       & $3.27$\,\tiny{$\pm$1.83} & $1.70$\,\tiny{$\pm$1.11} & $8.93$\,\tiny{$\pm$3.27} & $3.55$\,\tiny{$\pm$2.20} & $4.89$\,\tiny{$\pm$3.22} & $2.26$\,\tiny{$\pm$1.85} \\
\gls{umnist}     & $11.61$\,\tiny{$\pm$6.03} & $10.34$\,\tiny{$\pm$5.80} & $7.26$\,\tiny{$\pm$4.85} & $1.62$\,\tiny{$\pm$1.25} & $5.85$\,\tiny{$\pm$3.59} & $2.45$\,\tiny{$\pm$1.74} \\
\bottomrule
\end{tabular}%
}
\end{table}

These two diagnostics together indicate that, despite the absence of explicit calibration, the implicit alignment is sufficient for argmax routing.

\section{Training-Time Comparison}
\label{app:training_time}

Table~\ref{tab:training_time} reports wall-clock training time per method across the five benchmarks, measured on a single H100 GPU under identical hardware.

The two-stage methods that retrain a full robust model on top of ERM sit at the heavy end. GEORGE roughly doubles ERM on every dataset because its GDRO stage runs the same epoch budget. \gls{jtt} on \gls{waterbirds} takes $567$~minutes because its Stage~2 upweights the error set by $\lambda_\text{up}{=}50$, inflating the effective training set. \gls{exmap} and \gls{eiil} on \gls{isic} are dominated by their ResNet-50 ERM stage rather than by the LRP or environment-inference modules. \gls{dfr}'s last-layer retrain is effectively free, so its reported total is dominated by the ERM stage.

\gls{nct}'s training cost stays low for three reasons. Iteration~1 runs for only 1--3 epochs on every dataset except \gls{celeba}, since its purpose is hard-sample identification rather than convergence~\citep{liu2021just}. Phase~1 of iterations~2--3 trains only the new head layers ($\approx 1\%$ of parameters) and skips the backbone backward pass. The node sampler further shrinks each iteration-2/3 epoch on \gls{celeba} (downsampling to the smallest node size) and on \gls{umnist}/\gls{cmnist} (geomean sampling); \gls{waterbirds} and \gls{isic} use class-weighted BCE on full data, so they benefit only from the first two factors. The combined effect is visible in the table: on \gls{celeba}, \gls{nct} costs $41.6$~minutes against \gls{dfr}'s $148$, and on \gls{isic} it costs $106$~minutes against \gls{exmap}'s $177.7$ and \gls{eiil}'s $236$.

\begin{table}[h]
\centering
\caption{\textbf{Wall-clock training time} (minutes, single H100 GPU). Numbers are total end-to-end wall-clock time on the published per-dataset configuration. \gls{erm} and Group DRO are paper-sourced and not retrained in our pipeline. ``--'' indicates baselines not adapted to that benchmark.}
\label{tab:training_time}
\small
\setlength{\tabcolsep}{6pt}
\begin{tabular}{l ccccc}
\toprule
\textbf{Method} & \textbf{\gls{waterbirds}} & \textbf{\gls{celeba}} & \textbf{\gls{isic}} & \textbf{\gls{umnist}} & \textbf{\gls{cmnist}} \\
\midrule
GEORGE         & $112.9$ & $337.4$ & $130.3$ & $16.1$ & $19.1$ \\
JTT            & $567.4$ & --      & --      & $23.5$ & $5.4$  \\
CnC            & $29.2$  & --      & $251.8$      & $13.5$ & $61.7$ \\
DFR            & $19.7$  & $148.2$ & $64.6$  & $3.1$  & -- \\
EIIL           & $14.3$  & $47.7$  & $236.2$ & $6.5$  & $5.4$  \\
ExMap          & $20.5$  & $70.4$  & $177.7$ & $3.5$  & --  \\
\midrule
\textbf{\gls{nct} (Ours)} & $28.1$ & $41.6$ & $105.8$ & $2.6$ & $3.2$ \\
\bottomrule
\end{tabular}
\end{table}

\section{Depth-Selection Diagnostics}
\label{app:depth_diagnostics}

\subsection{Per-Iteration Test Performance}
\label{app:depth_per_iter}

Table~\ref{tab:depth_ablation} gives the per-iteration test \gls{wga} and average accuracy across the five benchmarks.

\begin{table}[h]
\centering
\caption{\textbf{Impact of hierarchy depth.} Worst-group accuracy (WGA) and average accuracy at each iteration of the hierarchy. Best \gls{wga} and average accuracy are highlighted in bold.}
\label{tab:depth_ablation}
\setlength{\tabcolsep}{4pt}
\resizebox{\textwidth}{!}{%
\begin{tabular}{l cc cc cc}
\toprule
\multirow{2}{*}{\textbf{Dataset}} & \multicolumn{2}{c}{\textbf{Iteration 1}} & \multicolumn{2}{c}{\textbf{Iteration 2}} & \multicolumn{2}{c}{\textbf{Iteration 3}} \\
\cmidrule(lr){2-3} \cmidrule(lr){4-5} \cmidrule(lr){6-7}
 & \textbf{WGA} & \textbf{Avg} & \textbf{WGA} & \textbf{Avg} & \textbf{WGA} & \textbf{Avg} \\
\midrule
\gls{waterbirds}       & $18.6$\,\tiny{$\pm$4.3} & $59.9$\,\tiny{$\pm$1.2} & $\mathbf{87.8}$\,\tiny{$\pm$1.0} & $\mathbf{92.7}$\,\tiny{$\pm$0.6} & $86.8$\,\tiny{$\pm$1.9} & $\mathbf{92.7}$\,\tiny{$\pm$0.8} \\
\gls{celeba}           & $23.7$\,\tiny{$\pm$2.2} & $85.3$\,\tiny{$\pm$1.2} & $\mathbf{86.1}$\,\tiny{$\pm$0.9} & $88.2$\,\tiny{$\pm$0.9} & $82.0$\,\tiny{$\pm$5.4} & $\mathbf{90.0}$\,\tiny{$\pm$1.2} \\
\midrule
\gls{isic} (Non-Patch) & $0.812$\,\tiny{$\pm$0.001} & $0.899$\,\tiny{$\pm$0.001} & $0.916$\,\tiny{$\pm$0.004} & $0.953$\,\tiny{$\pm$0.002} & $\mathbf{0.925}$\,\tiny{$\pm$0.003} & $\mathbf{0.960}$\,\tiny{$\pm$0.002} \\
\gls{isic} (Histopath) & $0.702$\,\tiny{$\pm$0.001} & $0.899$\,\tiny{$\pm$0.001} & $0.871$\,\tiny{$\pm$0.007} & $0.953$\,\tiny{$\pm$0.002} & $\mathbf{0.881}$\,\tiny{$\pm$0.004} & $\mathbf{0.960}$\,\tiny{$\pm$0.002} \\
\midrule
\gls{umnist}           & $80.8$\,\tiny{$\pm$8.5} & $96.6$\,\tiny{$\pm$1.0} & $\mathbf{94.4}$\,\tiny{$\pm$1.1} & $97.2$\,\tiny{$\pm$0.6} & $93.0$\,\tiny{$\pm$2.6} & $\mathbf{97.5}$\,\tiny{$\pm$0.4} \\
\gls{cmnist}           & $0.0$\,\tiny{$\pm$0.0}  & $20.7$\,\tiny{$\pm$1.0} & $\mathbf{72.3}$\,\tiny{$\pm$3.0} & $87.6$\,\tiny{$\pm$1.5} & $71.9$\,\tiny{$\pm$2.8} & $\mathbf{87.7}$\,\tiny{$\pm$2.1} \\
\bottomrule
\end{tabular}%
}
\end{table}

The optimal hierarchy depth is dataset-dependent and, on \gls{waterbirds} and \gls{cmnist}, even seed-dependent. The depth-selection rule of Section~\ref{sec:depth_control} adapts \emph{per seed}: it splits 4-1 in favor of iteration~2 on \gls{celeba}, picks iteration~3 for four of five seeds on \gls{isic}, and splits 3-2 in favor of iteration~2 on \gls{waterbirds}, 3-2 in favor of iteration~3 on \gls{umnist}, and 4-1 in favor of iteration~2 on \gls{cmnist}, all without consulting ground-truth groups.

\subsection{Validation Proxy Values}
\label{app:depth_proxy_values}

Table~\ref{tab:depth_diagnostics} reports the validation proxy values (pseudo-WGA, or pseudo-AUROC for \gls{isic}) consumed by the depth-selection rule (\cref{eq:depth_rule}). The iter-2-to-iter-3 drop in the validation proxy is largest on \gls{celeba}, small on \gls{waterbirds} and \gls{cmnist}, and reversed on \gls{isic}, mirroring the per-seed selection counts above.

\begin{table}[h]
\centering
\caption{\textbf{Validation pseudo-WGA per iteration} (mean over five seeds). \Gls{isic} reports pseudo-AUROC instead of pseudo-WGA, matching its main-text metric. The depth-selection rule stops at the first iteration whose proxy value falls more than $z\sqrt{p^{\star}(1-p^{\star})/n_{\text{worst}}}$ below the running best ($z=1.96$).}
\label{tab:depth_diagnostics}
\begin{tabular}{l cc}
\toprule
\textbf{Dataset} & \textbf{Iter 2} & \textbf{Iter 3} \\
\midrule
\gls{waterbirds} & $0.821$\,\tiny{$\pm$0.015} & $0.793$\,\tiny{$\pm$0.043} \\
\gls{celeba}     & $0.702$\,\tiny{$\pm$0.047} & $0.657$\,\tiny{$\pm$0.038} \\
\gls{isic} (pseudo-AUROC) & $0.937$\,\tiny{$\pm$0.011} & $0.945$\,\tiny{$\pm$0.005} \\
\gls{umnist}     & $0.781$\,\tiny{$\pm$0.028} & $0.778$\,\tiny{$\pm$0.024} \\
\gls{cmnist}     & $0.792$\,\tiny{$\pm$0.029} & $0.779$\,\tiny{$\pm$0.047} \\
\bottomrule
\end{tabular}
\end{table}

\subsection{Proxy Quality: Pseudo-WGA vs.\ True-WGA}
\label{app:pwga_correlation}

The depth-selection rule and the within-iteration early-stopping criterion both rely on $\mathrm{pWGA}_2$ (Section~\ref{sec:depth_control}), computed from the model's own routing labels rather than ground-truth groups. Its usefulness therefore rests on whether it tracks true \gls{wga}. Table~\ref{tab:pwga_corr} reports the Spearman rank correlation $\rho$ between the validation $\mathrm{pWGA}_2$ trajectory and the held-out true \gls{wga} across training epochs (iterations~2 and~3 combined). The correlation is high on \gls{waterbirds}, \gls{celeba}, and \gls{cmnist} ($\rho = 0.87$, $0.92$, and $0.90$), moderate on \gls{isic} ($\rho = 0.71$), and weakest on \gls{umnist} ($\rho = 0.43$).

\paragraph{On \gls{umnist}'s lower proxy correlation.} The weaker Spearman on \gls{umnist} ($\rho = 0.43$) reflects mild simplicity bias on this benchmark: \gls{erm} alone reaches $93.9\%$ \gls{wga} (Table~\ref{tab:main_results}), above \gls{jtt} and \gls{cnc}, so iteration~1 already classifies many minority digit-$8$ samples correctly. The hard $5$--$9$ leaf still captures most digit-$8$ samples at test time ($71.4\%$, Section~\ref{sec:structural_interpretability}), but a fraction is routed to the easy branch instead. This mixing means the hard leaf no longer cleanly corresponds to the minority subgroup, lowering $\mathrm{pWGA}_2$'s rank correlation with oracle \gls{wga}. The method itself remains effective: \gls{umnist} achieves $93.7\%$ test \gls{wga}, suggesting $\mathrm{pWGA}_2$ is robust to moderate misrouting---it still selects a depth that delivers strong worst-group performance even when its rank fidelity drops.

\begin{table}[h]
\centering
\caption{\textbf{$\mathrm{pWGA}_2$ proxy quality.} Spearman correlation between the validation $\mathrm{pWGA}_2$ trajectory and held-out true \gls{wga} across training (mean $\pm$ std over five seeds). Values close to $1$ indicate the proxy ranks checkpoints the same way the oracle would.}
\label{tab:pwga_corr}
\small
\setlength{\tabcolsep}{6pt}
\begin{tabular}{l c}
\toprule
\textbf{Dataset} & \textbf{Spearman $\rho$} \\
\midrule
\gls{waterbirds} & $0.87$\,\tiny{$\pm$0.07} \\
\gls{celeba}     & $0.92$\,\tiny{$\pm$0.04} \\
\gls{isic}       & $0.71$\,\tiny{$\pm$0.04} \\
\gls{umnist}     & $0.43$\,\tiny{$\pm$0.21} \\
\gls{cmnist}     & $0.90$\,\tiny{$\pm$0.08} \\
\bottomrule
\end{tabular}
\end{table}

\section{Node Sampling Strategies}
\label{app:sampling}

From iteration~2 onwards the leaves are imbalanced by construction (hard children always smaller than their easy siblings), and the natural class distribution of \gls{waterbirds}, \gls{celeba}, and \gls{cmnist} is also imbalanced at iteration~1. We support three rebalancing strategies; this appendix gives the precise definition of each, the per-dataset assignment, and a sensitivity comparison.

\subsection{Definitions.} 

Let $K^{(t)}$ be the number of leaves at iteration $t$ and $n_j$ the count of training samples assigned to node $j$.

\begin{itemize}
    \item \textbf{Class weights:} Sample-level loss is multiplied by $w_j = N / (K^{(t)} \cdot n_j)$, capped at a per-dataset maximum. Iteration~1 uses uniform weights; from iteration~2 onwards $w_j$ is recomputed per epoch from the current pseudo-label assignment.
    \item \textbf{Downsample:} At every epoch we draw a fresh subsample of size $\min_j n_j$ from each node. The epoch length is $K^{(t)} \cdot \min_j n_j$. Class weighting is disabled because every leaf already contributes equally.
    \item \textbf{Geomean:} At every epoch we resample each node to size $\bar n = (\prod_j n_j)^{1/K^{(t)}}$. Majority nodes are downsampled; minority nodes are oversampled with replacement. The geometric mean minimizes the maximum stretch factor across nodes and bounds the variance introduced by oversampling.
\end{itemize}

\subsection{Per-Dataset Assignment.}

\paragraph{Waterbirds and ISIC.} Training sets are small (4{,}795 and 19{,}124 samples), and after iteration~2 the rarest hard child of \gls{waterbirds} contains on the order of $10^{2}$ samples. Downsample would collapse the epoch to that count per leaf and starve every node, and geomean still trims the overall sample count. Class-weighted \gls{bce} keeps every sample available and produces the most stable \gls{wga} on both.

\paragraph{CelebA.} With 162{,}770 training samples, even the rarest hard child after iteration~2 retains enough examples for stable optimisation under downsample. Class weighting destabilises training because the imbalance ratio is very large (the largest weight cap is hit on every minority leaf), and geomean introduces variance through the large oversampling factor needed on hard children. Downsample balances batches without inflating any single sample.

\paragraph{UMNIST and CMNIST.} Totals are large (43{,}542 and 60{,}000 samples) but hard children are extremely sparse---digit~`8' has 234 training samples on \gls{umnist}, and \gls{cmnist}'s bias-conflicting color groups are even smaller. Class weighting produces extreme per-sample weights, while downsample starves all nodes; geomean balances the two, moving every node toward $\bar n$ with bounded oversampling on the small leaves.

\section{Component Ablations}
\label{app:component_ablations}

We isolate four components: trainable scope, head architecture, sparse-node merging and the auxiliary-loss weight. For every cell we pick the better of iteration~2 and~3 per seed (test \gls{wga}, or overall \gls{auroc} for \gls{isic}) and average across seeds. The default-configuration column reuses the paper's main-table seeds; non-default columns are run on a separate three-seed sweep.

\subsection{Trainable Scope}
\label{app:abl_trainable}

Each iteration is normally split into two phases: Phase~1 trains the new heads only with the backbone and parent layers frozen, then Phase~2 fine-tunes the entire network. Table~\ref{tab:ablation_trainable} compares this standard schedule against two extremes: \emph{head-only}, which keeps the backbone and parent layers frozen throughout, and \emph{full fine-tuning}, which skips Phase~1 entirely.

Head-only training collapses on every dataset---the pretrained backbone alone cannot resolve the hard children. Full fine-tuning is competitive with the standard schedule on \gls{waterbirds}, \gls{celeba}, \gls{umnist}, and \gls{isic}, where the dataset's iter-1 Phase~1 ratio is already near zero, but loses $10$~pp on \gls{cmnist}.

\begin{table}[h]
\caption{\textbf{Trainable-scope ablation.} Standard two-phase training vs.\ head-only and full fine-tuning.}
\label{tab:ablation_trainable}
\centering
\small
\setlength{\tabcolsep}{6pt}
\begin{tabular}{lccc}
\toprule
\textbf{Dataset} & \textbf{Standard} & \textbf{Head-Only} & \textbf{Full FT} \\
\midrule
\gls{waterbirds} (WGA)     & \textbf{88.0\,\tiny{$\pm$0.9} } & 42.3\,\tiny{$\pm$4.8}  & 87.2\,\tiny{$\pm$2.9} \\
\gls{celeba} (WGA)         & \textbf{86.1\,\tiny{$\pm$0.9}}  & 45.2\,\tiny{$\pm$3.6}  & 85.2\,\tiny{$\pm$0.9} \\
\gls{cmnist} (WGA)         & \textbf{72.9\,\tiny{$\pm$2.8}}  & 0.8\,\tiny{$\pm$1.1}   & 62.8\,\tiny{$\pm$18.1} \\
\gls{umnist} (WGA)         & \textbf{93.7\,\tiny{$\pm$2.5}}  & 57.6\,\tiny{$\pm$12.3}  & 90.1\,\tiny{$\pm$8.0} \\
\gls{isic} (Overall AUROC) & \textbf{.959\,\tiny{$\pm$.001}} & .900\,\tiny{$\pm$.004} & .949\,\tiny{$\pm$.012} \\
\bottomrule
\end{tabular}
\end{table}

\subsection{Head Architecture}
\label{app:abl_heads}

At iteration~2 and beyond, easy and hard children of a split can use different head architectures. Table~\ref{tab:abl_split} compares the asymmetric default (linear easy, MLP hard) with three alternatives: both children linear, both children MLP, and the reversed asymmetry (MLP easy, linear hard).

The asymmetric configuration is the strongest choice on \gls{celeba} and \gls{cmnist}, and within seed noise of the best alternative on \gls{waterbirds}, \gls{umnist}, and \gls{isic}. The reversed asymmetry consistently underperforms or matches the standard order, confirming that the additional capacity belongs on the hard branch where the conflicting features live, not on the easy branch where the parent has already done most of the work. We retain the default for its principled motivation: easy children classify samples the parent already gets right and need little capacity, while hard children carry the conflicting features that motivate a richer block.

\begin{table}[h]
\centering
\caption{\textbf{Head-architecture ablation.} \emph{Linear} = both children linear; \emph{MLP} = both children MLP; \emph{Asymmetric} = linear easy / MLP hard (default); \emph{Reversed} = MLP easy / linear hard.}
\label{tab:abl_split}
\small
\setlength{\tabcolsep}{6pt}
\begin{tabular}{l cccc}
\toprule
\textbf{Dataset} & \textbf{Linear} & \textbf{MLP} & \textbf{Asymmetric} & \textbf{Reversed} \\
\midrule
\gls{waterbirds} (WGA)     & 86.0\,\tiny{$\pm$0.7}  & \textbf{88.6\,\tiny{$\pm$0.4}}  & 88.0\,\tiny{$\pm$0.9}  & 87.8\,\tiny{$\pm$1.4} \\
\gls{celeba} (WGA)         & 83.8\,\tiny{$\pm$2.3}  & 85.3\,\tiny{$\pm$1.4}  & \textbf{86.1\,\tiny{$\pm$0.9}}  & 84.2\,\tiny{$\pm$2.7} \\
\gls{cmnist} (WGA)         & 59.0\,\tiny{$\pm$20.4} & 68.2\,\tiny{$\pm$3.8}  & \textbf{72.9\,\tiny{$\pm$2.8}}  & 66.2\,\tiny{$\pm$8.5} \\
\gls{umnist} (WGA)         & 92.0\,\tiny{$\pm$4.8}  & 92.9\,\tiny{$\pm$3.6}  & \textbf{93.7\,\tiny{$\pm$2.5}}  & 83.4\,\tiny{$\pm$4.4} \\
\gls{isic} (Overall AUROC) & .955\,\tiny{$\pm$.003} & .959\,\tiny{$\pm$.002} & \textbf{.959\,\tiny{$\pm$.001}} & .957\,\tiny{$\pm$.005} \\
\bottomrule
\end{tabular}
\end{table}

\subsection{Sparse-Node Merging}
\label{app:abl_merging}

Hard children of an iteration-2 split can inherit very few samples. The sparse-node merging rule folds any hard child with fewer than $m_{\min}$ training samples into its easy sibling. Table~\ref{tab:abl_minsamples} compares disabling the rule ($m_{\min}=0$) against the default $m_{\min}$, separately at iteration~2 and iteration~3.

The threshold only ever fires on \gls{waterbirds} at iteration~3, where it converts a $0.7$~pp regression into a small recovery and reduces seed variance. On every other dataset the iteration-2 and iteration-3 hard children stay above the cutoff (the per-dataset $m_{\min}=20$ is set conservatively so the rule fires only when a hard branch would otherwise be optimised on a handful of samples), so the two columns are identical by construction. We retain the rule as a safety net for deeper hierarchies.

\begin{table}[h]
\centering
\caption{\textbf{Sparse-node merging ablation.} ``Disabled'' = $m_{\min}=0$. ``Default'' uses the per-dataset $m_{\min}$ from the main results. Values are reported separately at iteration~2 and iteration~3 (3-seed mean $\pm$ std).}
\label{tab:abl_minsamples}
\small
\setlength{\tabcolsep}{5pt}
\begin{tabular}{l cc cc}
\toprule
 & \multicolumn{2}{c}{\textbf{Iteration 2}} & \multicolumn{2}{c}{\textbf{Iteration 3}} \\
\cmidrule(lr){2-3} \cmidrule(lr){4-5}
\textbf{Dataset} & \textbf{Disabled} & \textbf{Default} & \textbf{Disabled} & \textbf{Default} \\
\midrule
\gls{waterbirds} (WGA)     & 88.4\,\tiny{$\pm$0.4}  & 88.4\,\tiny{$\pm$0.4}  & 85.5\,\tiny{$\pm$2.7}  & \textbf{86.2\,\tiny{$\pm$2.1}} \\
\gls{celeba} (WGA)         & 86.1\,\tiny{$\pm$0.9}  & 86.1\,\tiny{$\pm$0.9}  & 82.0\,\tiny{$\pm$5.4}  & 82.0\,\tiny{$\pm$5.4} \\
\gls{cmnist} (WGA)         & 54.5\,\tiny{$\pm$30.9} & 54.5\,\tiny{$\pm$30.9} & 58.3\,\tiny{$\pm$19.0} & 58.3\,\tiny{$\pm$19.0} \\
\gls{umnist} (WGA)         & 94.3\,\tiny{$\pm$1.5}  & 94.3\,\tiny{$\pm$1.5}  & 93.5\,\tiny{$\pm$3.4}  & 93.5\,\tiny{$\pm$3.4} \\
\gls{isic} (Overall AUROC) & .950\,\tiny{$\pm$.005} & .950\,\tiny{$\pm$.005} & .958\,\tiny{$\pm$.004} & .958\,\tiny{$\pm$.004} \\
\bottomrule
\end{tabular}
\end{table}

\subsection{Auxiliary-Loss Weight}
\label{app:abl_aux}

\begin{table}[h]
\centering
\caption{\textbf{Auxiliary-loss coefficient sweep.} \gls{wga} (or Overall \gls{auroc} for \gls{isic}) as $\lambda_{\text{aux}}$ varies.}
\label{tab:aux_sweep}
\small
\setlength{\tabcolsep}{4pt}
\resizebox{\textwidth}{!}{%
\begin{tabular}{l cccccc}
\toprule
$\lambda_{\text{aux}}$ & 0.0 & 0.3 & 0.5 & 0.7 & 1.0 & 2.0 \\
\midrule
\gls{waterbirds} (WGA)     & 86.4\,\tiny{$\pm$0.5}  & 88.6\,\tiny{$\pm$1.5}  & 87.7\,\tiny{$\pm$2.1}  & \textbf{88.7\,\tiny{$\pm$1.1}}  & 88.0\,\tiny{$\pm$0.9}  & 88.4\,\tiny{$\pm$1.6} \\
\gls{celeba} (WGA)         & 68.5\,\tiny{$\pm$3.4}  & 74.3\,\tiny{$\pm$4.0}  & 76.8\,\tiny{$\pm$3.9}  & 78.7\,\tiny{$\pm$4.2}  & 81.5\,\tiny{$\pm$2.8}  & \textbf{86.1\,\tiny{$\pm$0.9}} \\
\gls{cmnist} (WGA)         & 50.6\,\tiny{$\pm$27.9} & 53.3\,\tiny{$\pm$27.2} & 61.2\,\tiny{$\pm$16.0} & 53.0\,\tiny{$\pm$28.5} & \textbf{72.9\,\tiny{$\pm$2.8}}  & 65.7\,\tiny{$\pm$10.5} \\
\gls{umnist} (WGA)         & 89.1\,\tiny{$\pm$7.2}  & 93.5\,\tiny{$\pm$2.3}  & 93.7\,\tiny{$\pm$2.5}  & 92.9\,\tiny{$\pm$4.0}  & \textbf{95.4\,\tiny{$\pm$1.8}}  & 92.5\,\tiny{$\pm$5.0} \\
\gls{isic} (Overall AUROC) & .955\,\tiny{$\pm$.004} & .955\,\tiny{$\pm$.002} & .959\,\tiny{$\pm$.003} & .958\,\tiny{$\pm$.001} & \textbf{.959\,\tiny{$\pm$.001}} & .957\,\tiny{$\pm$.001} \\
\bottomrule
\end{tabular}%
}
\end{table}

The auxiliary term $\lambda_{\text{aux}}\mathcal{L}^{(t-1)}$ balances child specialization against parent stability. Table~\ref{tab:aux_sweep} sweeps $\lambda_{\text{aux}} \in \{0, 0.3, 0.5, 0.7, 1.0, 2.0\}$ on all five datasets.

Removing the auxiliary loss costs $1.6$~pp \gls{wga} on \gls{waterbirds}, $17.6$~pp on \gls{celeba}, and $22$~pp on \gls{cmnist}; \gls{umnist} and \gls{isic} are within seed noise. The empirical optimum is dataset-dependent and does not always match the chosen default, but every dataset benefits from at least one nonzero coefficient.

\section{Theoretical Proofs}
\label{app:proofs}

\subsection{Proof of Theorem~\ref{thm:main} (Minority Enrichment)}
\label{app:thm_main_proof}

We give the full derivation of Theorem~\ref{thm:main} deferred from Section~\ref{sec:theory}.

\begin{proof}
Since $\ell_i^{(1)} = y_i$, we write $y$ for the true label throughout. Under simplicity bias (Assumption~\ref{assump:simplicity}), the model predicts $\hat{\ell}^{(1)} \approx a$. We derive the probability of error for each group:
\begin{itemize}
    \item Majority groups ($a=y$): The model predicts $\hat{\ell}^{(1)} \approx a = y$, which is correct. An error occurs only with probability $\epsilon$:
    \begin{equation*}
    P(\hat{\ell}^{(1)} \neq y \mid a = y) = P(\hat{\ell}^{(1)} \neq a) = \epsilon
    \end{equation*}
    \item Minority groups ($a \neq y$): The model predicts $\hat{\ell}^{(1)} \approx a \neq y$, which is incorrect. An error occurs with high probability $1-\epsilon$:
    \begin{equation*}
    P(\hat{\ell}^{(1)} \neq y \mid a \neq y) = P(\hat{\ell}^{(1)} = a) = 1 - \epsilon
    \end{equation*}
\end{itemize}
We apply Bayes' theorem to find the posterior probability that a misclassified sample belongs to a minority group:
\begin{equation*}
\begin{split}
P(a \neq y \mid \hat{\ell}^{(1)} \neq y) &= \frac{P(\hat{\ell}^{(1)} \neq y \mid a \neq y) P(a \neq y)}{P(\hat{\ell}^{(1)} \neq y)} \\
&= \frac{(1-\epsilon)(1-\rho)}{\epsilon\rho + (1-\epsilon)(1-\rho)}
\end{split}
\end{equation*}
\end{proof}

\subsection{Constructive Lower Bound for \texorpdfstring{$\Delta$}{Δ}}
\label{app:delta_bound}

We instantiate the approximation gap of Theorem~\ref{thm:approx} on a tractable
Gaussian feature model and derive a closed-form lower bound that grows with the
spurious correlation strength.

\subsubsection{Setup}

Let $y, a \in \{0,1\}$ denote the class label and spurious attribute,
respectively, and let $\Phi$ denote the standard normal CDF. Each input has
features $\mathbf{z} = (z_s, z_c)$ drawn according to
\begin{align}
  z_s &\sim \mathcal{N}\!\bigl((2a-1)\mu_s,\; 1\bigr), \label{eq:zs}\\
  z_c &\sim \mathcal{N}\!\bigl((2y-1)\mu_c,\; 1\bigr), \label{eq:zc}
\end{align}
where $\mu_s, \mu_c > 0$ control the magnitudes of the spurious and core
features, and the map $(2a-1)$ sends $\{0,1\} \to \{-1,+1\}$. We assume equal
class priors and that $z_s, z_c$ are conditionally independent given $(y,a)$,
so the class-conditional covariance is the identity. A linear classifier
$w = (w_1, w_2)$ predicts the label via $\operatorname{sign}(w_1 z_s + w_2 z_c)$.

\subsubsection{Specialist Classifiers}

\paragraph{Easy subgroup ($a = y$).}
Both features align with the label. For two Gaussian classes with equal priors
and shared identity covariance, the Bayes-optimal linear direction is the
difference of class means (the standard LDA result, with the shared covariance
eliminating the quadratic term). The two class-conditional means on the easy
subgroup are $(\mu_s, \mu_c)$ for $y=1$ and $(-\mu_s, -\mu_c)$ for $y=0$, so
\[
  (\mu_s, \mu_c) - (-\mu_s, -\mu_c) \;=\; 2(\mu_s, \mu_c)
  \;\propto\; (\mu_s, \mu_c),
\]
giving the optimal weight vector $w_E \propto (\mu_s, \mu_c)$.

\paragraph{Hard subgroup ($a \neq y$).}
The spurious feature flips sign relative to the label: when $y=1$ we have
$a=0$, so $z_s \sim \mathcal{N}(-\mu_s, 1)$ instead of $\mathcal{N}(\mu_s, 1)$.
The class-conditional means become $(-\mu_s, \mu_c)$ for $y=1$ and
$(\mu_s, -\mu_c)$ for $y=0$, so by the same LDA argument the optimal weight
vector is $w_H \propto (-\mu_s, \mu_c)$.

\paragraph{Specialist risk.}
Take $w_E = (\mu_s, \mu_c)$. On the easy subgroup with $y=1$, the score
$w_E \cdot \mathbf{z} = \mu_s z_s + \mu_c z_c$ is a linear combination of
independent Gaussians, hence Gaussian itself with
\[
  \mathbb{E}[w_E \cdot \mathbf{z}] = \mu_s^2 + \mu_c^2,
  \qquad
  \operatorname{Var}(w_E \cdot \mathbf{z}) = \mu_s^2 + \mu_c^2.
\]
A misclassification occurs when this score is negative, so
\begin{equation}
  R_E^{\star}
  \;=\; P\!\left(w_E \cdot \mathbf{z} < 0 \,\big|\, y=1\right)
  \;=\; \Phi\!\left(\frac{0 - (\mu_s^2 + \mu_c^2)}{\sqrt{\mu_s^2 + \mu_c^2}}\right)
  \;=\; \Phi\!\left(-\sqrt{\mu_s^2 + \mu_c^2}\right).
  \label{eq:specialist-risk}
\end{equation}
The case $y=0$ is symmetric and yields the same value. For the hard subgroup,
$w_H = (-\mu_s, \mu_c)$ pairs with $z_s$ whose mean is negated relative to the
easy case, so $w_H \cdot \mathbf{z}$ has the same Gaussian distribution as
$w_E \cdot \mathbf{z}$ on the easy subgroup, and $R_H^{\star} = R_E^{\star}$.

\subsubsection{Single Classifier}

A single linear classifier shares $(w_1, w_2)$ across both subgroups, giving
\begin{align}
  R_E(w) &\;=\; \Phi\!\left(-\frac{w_1 \mu_s + w_2 \mu_c}{\|w\|}\right),
  \label{eq:RE}\\
  R_H(w) &\;=\; \Phi\!\left(-\frac{-w_1 \mu_s + w_2 \mu_c}{\|w\|}\right).
  \label{eq:RH}
\end{align}
The worst-group risk is $\max(R_E(w), R_H(w))$, and we seek the $w$ that
minimises it.

\paragraph{The minimax is attained at $w_1 = 0$.}
The risks \eqref{eq:RE} and \eqref{eq:RH} are related by the symmetry
$w_1 \mapsto -w_1$, which swaps $R_E$ and $R_H$, so the worst-group risk is
symmetric in $w_1$. This motivates checking $w_1 = 0$, but symmetry alone does
not force the optimum onto the axis (a symmetric function may have off-axis
minima); we therefore verify the bound directly. Since $\Phi$ is monotone
increasing,
\[
  \max(R_E(w), R_H(w)) \;=\; \Phi\!\left(-\min(a_E, a_H)\right),
  \quad
  a_E \;=\; \frac{w_1\mu_s + w_2\mu_c}{\|w\|},\;\;
  a_H \;=\; \frac{-w_1\mu_s + w_2\mu_c}{\|w\|}.
\]
Using the identity $\min(b + c,\, -b + c) = c - |b|$ with $b = w_1\mu_s$ and
$c = w_2\mu_c$, and bounding in three steps:
\begin{equation}
  \min(a_E, a_H)
  \;=\; \frac{w_2 \mu_c - |w_1|\mu_s}{\sqrt{w_1^2 + w_2^2}}
  \;\overset{(i)}{\leq}\; \frac{w_2 \mu_c}{\sqrt{w_1^2 + w_2^2}}
  \;\overset{(ii)}{\leq}\; \frac{|w_2| \mu_c}{\sqrt{w_1^2 + w_2^2}}
  \;\overset{(iii)}{\leq}\; \mu_c,
  \label{eq:min-bound}
\end{equation}
where (i) uses $|w_1|\mu_s \geq 0$, (ii) uses $w_2 \leq |w_2|$, and (iii) uses
$|w_2| \leq \|w\|$. Equality holds in (i) iff $w_1 = 0$, in (ii) iff
$w_2 \geq 0$, and in (iii) iff $w_1 = 0$; jointly, equality throughout requires
$w_1 = 0$ and $w_2 > 0$ (the strict sign on $w_2$ is needed for $\|w\| > 0$).

\paragraph{Resulting risk.}
From \eqref{eq:min-bound} we obtain $\max(R_E(w), R_H(w)) \geq \Phi(-\mu_c)$
for every $w \neq 0$, with equality iff $w_1 = 0$. The single classifier
therefore attains worst-group risk $\Phi(-\mu_c)$ and does so only by
discarding the spurious feature entirely.

\subsubsection{Closed-Form Gap}

Combining the two cases yields
\begin{equation}
  \boxed{\;
  \Delta \;=\; \Phi(-\mu_c) - \Phi\!\left(-\sqrt{\mu_s^2 + \mu_c^2}\right)
  \;>\; 0 \quad \text{whenever } \mu_s > 0,
  \;}
  \label{eq:delta-closed-form}
\end{equation}
since $\sqrt{\mu_s^2 + \mu_c^2} > \mu_c$ and $\Phi$ is monotone. The gap grows
with the spurious magnitude $\mu_s$: stronger spurious correlations make
structural separation more beneficial.

\paragraph{Numerical example.}
For a \gls{waterbirds}-like regime with $\mu_s = 2$ and $\mu_c = 1$,
\[
  \Phi(-1) \approx 0.159,
  \qquad
  \Phi(-\sqrt{5}) \approx 0.013,
\]
so $\Delta \approx 0.146$, i.e.\ a $14.6$ percentage-point gap.

\subsubsection{Interpretation}

JTT and similar reweighting methods upweight hard samples within a single
classifier; that classifier still faces the gap $\Delta$ in
\eqref{eq:delta-closed-form}. \Gls{nct}'s separate heads remove this gap.
Routing errors do not compound at inference because all leaf heads run in
parallel and the prediction is taken via $\arg\max$.

Hierarchy is particularly helpful because child heads receive parent
representations (Section~\ref{sec:architecture}), enabling progressive
specialisation that is more sample-efficient than a flat mixture discovering
both classes and subgroups simultaneously.

\section{Iteration-3 Routing Analysis}
\label{app:node_analysis}

\begin{figure*}[h]
    \centering
    \includegraphics[width=\linewidth]{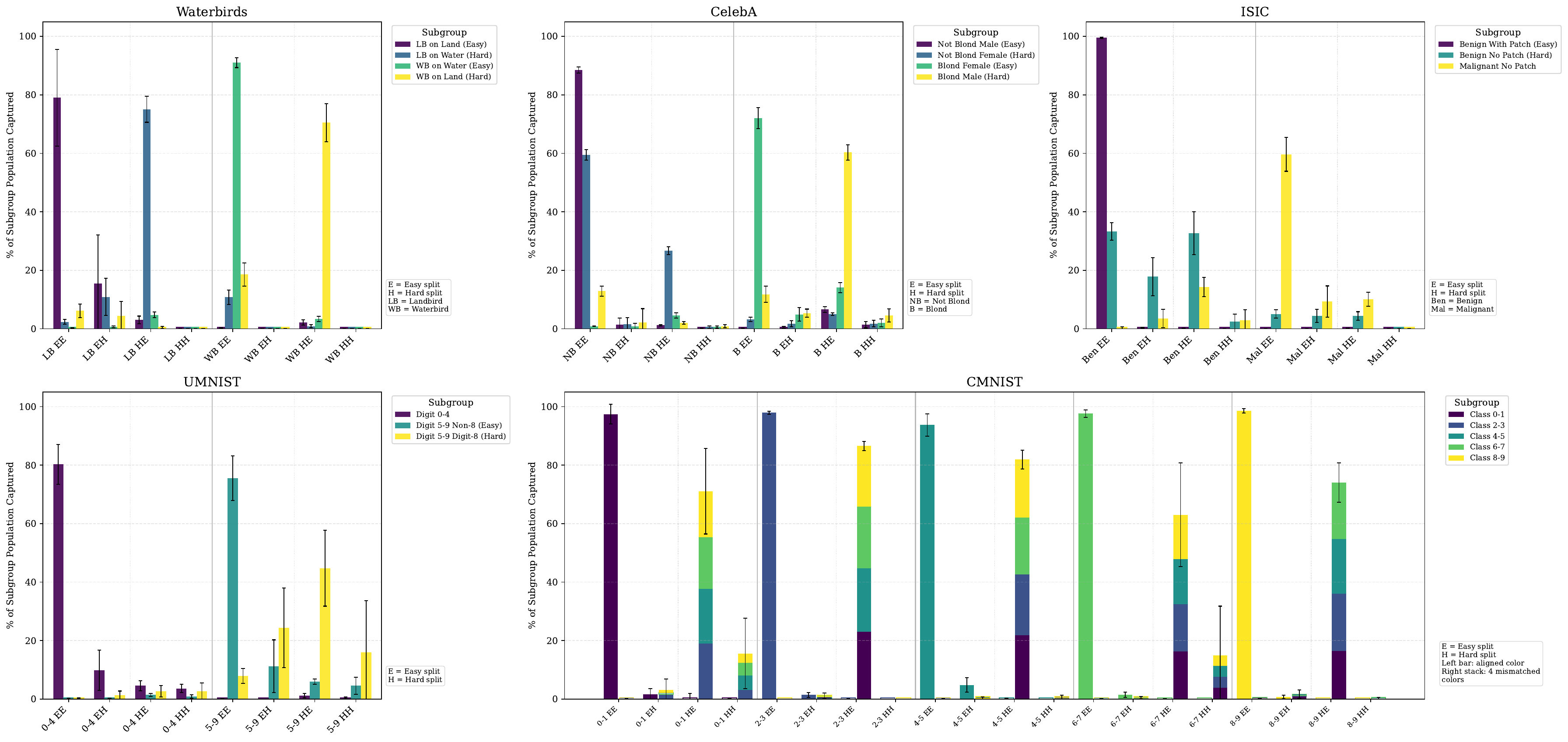}
    \caption{\textbf{Iteration-3 capture rates.} Bar height = \% of a subgroup's population routed to each leaf (mean $\pm$ std, five seeds). Iter-3 leaf labels: \textbf{EE} = easy~$\rightarrow$~easy, \textbf{EH} = easy~$\rightarrow$~hard, \textbf{HE} = hard~$\rightarrow$~easy, \textbf{HH} = hard~$\rightarrow$~hard. For \gls{cmnist}, each leaf shows the own-class matching-color subgroup (left bar) and the four own-class mismatched-color subgroups stacked (right bar).}
    \label{fig:node_dynamics_iter3}
\end{figure*}

Iteration~3 splits each iter-2 leaf into easy and hard children, yielding eight leaves per binary task and twenty leaves on \gls{cmnist}. Figure~\ref{fig:node_dynamics_iter3} reports the per-subgroup capture rate at each iter-3 leaf. The pattern is consistent across datasets: bias-aligned subgroups consolidate in the easy-easy (EE) leaf of their own class, and minority subgroups concentrate in their class's hard-easy (HE) leaf---the iter-2 hard child further refined by the additional split.

\textbf{Waterbirds.} The minority \textit{landbird-on-water} subgroup lands $75.1 \pm 4.5\%$ in LB:HE; the symmetric minority \textit{waterbird-on-land} lands $70.5 \pm 6.5\%$ in WB:HE. Majority subgroups stay in their EE leaves (LB-Land $79.0 \pm 16.5\%$, WB-Water $91.0 \pm 1.7\%$). Two leaves (LB:HH, WB:EH) are unused on every seed because the sparse-node merging rule (Section~\ref{sec:depth_control}) folded their iter-2 children into the easy siblings.

\textbf{CelebA.} The minority \textit{blond-male} subgroup concentrates in B:HE ($60.3 \pm 2.6\%$), with another $12.9\%$ in NB:EE. Majority \textit{not-blond-male} retains $88.5\%$ in NB:EE. \textit{Blond-female} consolidates in B:EE ($72.1\%$) with $14.1\%$ in B:HE.

\textbf{ISIC.} The patch shortcut is locked in: \textit{benign-with-patch} captures $99.5\%$ of its population in Ben:EE. The conflicting \textit{benign-no-patch} subgroup splits roughly evenly between Ben:EE ($33.3\%$) and Ben:HE ($32.7\%$). Malignant samples reach a malignant-class leaf at $\approx 79\%$ (Mal:EE $59.7\%$, Mal:EH $9.3\%$, Mal:HE $10.1\%$).

\textbf{UMNIST.} The undersampled \textit{digit-8} minority concentrates in 5-9:HE ($44.7 \pm 13.0\%$); combined with the other 5-9 hard descendants, the 5-9 hard branch holds around 85\% of digit-8 samples. Majority subgroups stay in their EE leaves (Digit 0-4 at $80.3\%$, Digit 5-9 Non-8 at $75.6\%$).

\textbf{CMNIST.} Across all five classes, bias-aligned subgroups consolidate in their own EE leaf at $97.1\%$ on average and stay within their true-class branch at $99.2\%$. Bias-conflicting subgroups (color $\neq$ digit class) reach a hard leaf of their true class at $83.0\%$, averaged across the twenty conflicting subgroups. The iter-2 to iter-3 transition therefore refines the within-class routing without leaking conflicting samples to the wrong digit branch.


\end{document}